\documentclass{article}
\usepackage[utf8]{inputenc}

\usepackage[margin=1in]{geometry}

\usepackage{microtype}
\usepackage{graphicx}
\usepackage{subfigure}
\usepackage{booktabs} %

\usepackage{amsmath}

\usepackage[colorlinks=true,allcolors=blue,hyperfootnotes=false]{hyperref}
\usepackage[capitalise]{cleveref}

\newenvironment{ack}{\subsection*{Acknowledgements}}

\title{Bandit Linear Control}

\author{%
Asaf Cassel%
\thanks{School of Computer Science, Tel Aviv University; \texttt{acassel@mail.tau.ac.il}.}
\and
Tomer Koren%
\thanks{School of Computer Science, Tel Aviv University; \texttt{tkoren@tauex.tau.ac.il}.}
}
\RequirePackage{mathtools}
\RequirePackage{delimset}

\newcommand{\floor}[2][*]{\delim\lfloor\rfloor#1{#2}}

\newcommand{\overbar}[1]{\mkern 2mu\overline{\mkern-2mu#1\mkern-2mu}\mkern 2mu}
\newcommand{\tran}{^{\mkern-1.5mu\mathsf{T}}}

\newcommand{\EE}[2][]{\mathbf{E}_{#1}{#2}}
\newcommand{\EEBrk}[2][*]{\mathbf{E}\delim{[}{]}#1{#2}}
\newcommand{\RR}[1][]{\mathbb{R}^{#1}}

\newcommand{\sphere}[1]{\mathcal{S}^{#1}}
\newcommand{\ball}[1]{\mathcal{B}^{#1}}

\DeclareMathOperator*{\argmin}{arg\,min}

\DeclarePairedDelimiterX\setDef[1]\lbrace\rbrace{#1}

\usepackage{amsthm}
\usepackage{amssymb}
\usepackage{thmtools}

\declaretheoremstyle[
	    spaceabove=\topsep, 
	    spacebelow=\topsep, 
	    bodyfont=\normalfont\itshape,
    ]{theorem}

\declaretheorem[style=theorem,name=Theorem]{theorem}

\declaretheoremstyle[
	    spaceabove=\topsep, 
	    spacebelow=\topsep, 
	    bodyfont=\normalfont,
    ]{definition}

\declaretheoremstyle[
        spaceabove=\topsep, 
        spacebelow=\topsep, 
        bodyfont=\normalfont,
        notefont=\normalfont\bfseries,
        notebraces={}{},
        qed=$\blacksquare$, 
    ]{proofstyle}
\declaretheorem[style=proofstyle,numbered=no,name=Proof]{proof}

\declaretheorem[style=theorem,sibling=theorem,name=Lemma]{lemma}

\declaretheorem[style=theorem,numbered=no,name=Theorem]{theorem*}
\declaretheorem[style=theorem,numbered=no,name=Lemma]{lemma*}
\declaretheorem[style=theorem,numbered=no,name=Corollary]{corollary*}
\declaretheorem[style=theorem,numbered=no,name=Proposition]{proposition*}
\declaretheorem[style=theorem,numbered=no,name=Claim]{claim*}
\declaretheorem[style=theorem,numbered=no,name=Fact]{fact*}
\declaretheorem[style=theorem,numbered=no,name=Observation]{observation*}
\declaretheorem[style=theorem,numbered=no,name=Conjecture]{conjecture*}

\declaretheorem[style=definition,sibling=theorem,name=Definition]{definition}

\declaretheorem[style=definition,numbered=no,name=Definition]{definition*}
\declaretheorem[style=definition,numbered=no,name=Remark]{remark*}
\declaretheorem[style=definition,numbered=no,name=Example]{example*}
\declaretheorem[style=definition,numbered=no,name=Question]{question*}
\declaretheorem[style=definition,numbered=no,name=Assumption]{assumption*}

\newcommand{\pSmooth}{\beta}
\newcommand{\pConvex}{\alpha}
\newcommand{\pLip}{L}

\newcommand{\noiseStd}{\sigma}

\newcommand{\Astar}{A_{\star}}
\newcommand{\Bstar}{B_{\star}}

\newcommand{\memLen}{H}
\newcommand{\rt}[1][t]{r_{#1}}
\newcommand{\rti}[1][t]{r_{#1}^{\brk
[s]{i}}}
\newcommand{\rtOf}[2][t]{r_{#1}^{\brk
[s]{#2}}}
\newcommand{\Mt}[1][t]{M_{#1}}
\newcommand{\Mti}[1][t]{M_{#1}^{\brk[s]{i}}}
\newcommand{\barMt}[1][\tau]{\overbar{M}_{#1}}
\newcommand{\barMti}[1][\tau]{\overbar{M}_{#1}^{\brk[s]{i}}}
\newcommand{\barMtOf}[2][\tau]{\overbar{M}_{#1}^{\brk[s]{#2}}}

\newcommand{\Ut}[1][\tau]{U_{#1}}
\newcommand{\Uti}[1][\tau]{U_{#1}^{\brk[s]{i}}}
\newcommand{\wt}[1][t]{w_{#1}}
\newcommand{\xt}[1][t]{x_{#1}}
\newcommand{\ut}[1][t]{u_{#1}}
\newcommand{\ct}[3][t]{c_{#1}\brk*{#2, #3}}
\newcommand{\bt}[1][t]{b_{#1}}

\newcommand{\isUpRound}[1][t]{\chi_{#1}}

\newcommand{\gHat}[1][t]{\hat{g}_{#1}}
\newcommand{\constraintD}{D}
\newcommand{\constraintDp}{D_{P_0}}
\newcommand{\bcoFeedbackBound}{\hat{C}}

\newcommand{\pMeanShift}[1][t]{\delta_{#1}}
\newcommand{\pPredVar}[1][t]{\vartheta_{#1}}
\newcommand{\bcoMemRegret}[2][H]{\mathcal{R}_{#1}\brk*{#2}}
\newcommand{\algRegretBound}[2][\mathcal{A}]{R_{#1}\brk*{#2}}
\newcommand{\fRange}{\mathcal{K}_{+}}
\newcommand{\fConstraint}{\mathcal{K}}
\newcommand{\bcoFeedback}[1][t]{\hat{f}_{#1}}

\newcommand{\funcSet}{\mathcal{F}}

\newcommand{\baseBCO}{\mathcal{A}}
\newcommand{\baseBCOof}[1]{\mathcal{A}\brk{#1}}

\newcommand{\controlRegret}{\mathcal{R}_{\mathcal{A}}(T)}
\newcommand{\controlPseudoRegret}{\overbar{\mathcal{R}}_{\mathcal{A}}(T)}

\newcommand{\cD}{D}
\newcommand{\cC}{C}
\newcommand{\cG}{G}

\newcommand{\MRange}{\mathcal{M}_{+}}
\newcommand{\MConstraint}{\mathcal{M}}
\newcommand{\MiConstraint}[1][i]{\mathcal{M}^{\brk[s]{#1}}}

\newcommand{\Dxu}{D_{x,u}}

\newcommand{\pSmoothF}{\pSmooth_f}
\newcommand{\pConvexF}{\pConvex_f}
\newcommand{\pLipF}{\pLip_f}

\newcommand{\vStack}[1]{\brk[s]{#1}_{\mathrm{vec}}}
\newcommand{\dM}{d_{\MConstraint}} 

\usepackage{enumitem}
\usepackage[numbers]{natbib}
\usepackage{algorithm}
\usepackage[noend]{algpseudocode}

\algnewcommand{\IfThenElse}[3]{%
  \State \algorithmicif\ #1\ \algorithmicthen\ #2\ \algorithmicelse\ #3}

\begin{document}
\maketitle

\begin{abstract}
We consider the problem of controlling a known linear dynamical system under stochastic noise, adversarially chosen costs, and bandit feedback. Unlike the full feedback setting where the entire cost function is revealed after each decision, here only the cost incurred by the learner is observed. We present a new and efficient algorithm that, for strongly convex and smooth costs, obtains regret that grows with the square root of the time horizon $T$. We also give extensions of this result to general convex, possibly non-smooth costs, and to non-stochastic system noise. A key component of our algorithm is a new technique for addressing bandit optimization of loss functions with memory.
\end{abstract}

\section{Introduction}

Reinforcement learning studies sequential decision making problems where a
learning agent repeatedly interacts with an environment and aims to improve her
strategy over time based on the received feedback. One of the most fundamental
tradeoffs in reinforcement learning theory is the \emph{exploration
vs.~exploitation tradeoff}, that arises whenever the learner observes only
partial feedback after each of her decisions, thus having to balance between
exploring new strategies and exploiting those that are already known to perform
well. The most basic and well-studied form of partial feedback is the so-called
``bandit'' feedback, where the learner only observes the cost of her chosen
action on each decision round, while obtaining no information about the
performance of other actions.

Traditionally, the environment dynamics in reinforcement learning are modeled as
a Markov Decision Process (MDP) with a finite number of possible states and
actions. The MDP model has been studied and analyzed in numerous different
settings and under various assumptions on the transition parameters, the nature
of the reward functions, and the feedback model.
Recently, a particular focus has been given to continuous state-action MDPs, and
in particular, to a specific family of models in classic control where the state
transition function is linear. 
Concretely, in \emph{linear control} the state evolution follows the linear dynamics:
\begin{align} \label{eq:linDynamics}
    x_{t+1}
    =
    \Astar x_t
    +
    \Bstar u_t
    +
    w_t,
\end{align}
where $x_t \in \RR[d], u_t \in \RR[k], w_t \in \RR[d]$ are respectively the
system state, action (control), and noise at round $t$, and $\Astar \in \RR[d
\times d], \Bstar \in \RR[d \times k]$ are the system parameters. The goal is to
minimize the total control costs with respect to cost function $c_t(x,u) :
\RR[d] \times \RR[k] \to \RR$ put forward on round $t$.

However, in contrast to the reinforcement learning literature, existing work on
learning in linear control largely assumes the full feedback model, where after
each decision round the learning agent observes the entire cost function $c_t$
used to assign costs on the same round. 
In fact, to the best of our knowledge, thus far linear control has not been
studied in the bandit setting, even in the special case where the costs are
generated stochastically over time.

\paragraph{Contributions.}

In this paper, we introduce and study the \emph{bandit linear control} problem,
where a learning agent has to control a known linear dynamical system (as in
\cref{eq:linDynamics}) under stochastic noise, adversarially chosen convex cost
functions, and bandit feedback. Namely, after each decision round the learner
only observes the incurred cost $c_t(x_t,u_t)$ as feedback. (We still assume,
however, that the state evolution is fully observable.)
For strongly convex and smooth cost functions, we present an efficient bandit
algorithm that achieves $\smash{ \tilde{O}(\sqrt{T}) }$ regret over $T$ decision
rounds, with a polynomial dependence on the natural problem parameters. This
result is optimal up to polylogarithmic factors as it matches the optimal regret
rate in the easier stationary (i.e., stateless) strongly convex and smooth
bandit optimization setting~\citep{shamir2013complexity,hazan2014bandit}.

The starting point of our algorithmic approach is an approximate
reparameterization of the online control problem due to
\cite{agarwal2019online,agarwal2019logarithmic}, called the Disturbance-Action
Policy. In this new parameterization, the control problem is cast in terms of
bounded memory convex loss functions, under which the cost of the learner on
each round depends explicitly only on her last few decisions rather than on the
entire history (this is thanks to strong stability conditions of the learned
policies~\cite{cohen2018online}).

As a key technical tool, we develop a new reduction technique for addressing
bandit convex optimization with bounded memory. While an analogous reduction has
been well established in the full feedback model
\cite{anava2015online,agarwal2019online,agarwal2019logarithmic}, its adaptation
to bandit feedback is far from straightforward. Indeed, loss functions with
memory in the bandit setting were
previously studied by \cite{arora2012online}, that showed a black-box reduction
via a mini-batching approach that, for an algorithm achieving $\smash{
O(T^{1/2}) }$ regret in the no-memory bandit setting, achieves $\smash{
O(T^{2/3}) }$ regret with memory. While this technique imposes very few
restrictions on the adversary, it degrades performance significantly even for
adversaries with bounded memory. In contrast, \cite{anava2015online} show that
the full-feedback setting enjoys nearly no degradation when the adversary's
memory is fixed and bounded. Our new technique establishes a similar lossless
reduction for bandit feedback with adversaries restricted to choosing
\emph{smooth} cost functions. 
Combining these ideas with standard techniques in (no-memory) bandit convex
optimization \cite{flaxman2004online,saha2011improved,hazan2014bandit} gives our
main result.

Our techniques readily extend to weakly convex costs with regret scaling as $\smash{ \tilde{O}(T^{2/3}) }$ in the smooth case and $\smash{ \tilde{O}(T^{3/4}) }$ without smoothness. Moreover, these hold even without the stochastic assumptions on the system noise $w_t$, which were only required in our analysis for preserving the strong convexity of the costs through the reduction to loss functions with memory. We defer further details on these extensions to later sections and choose to focus first on the more challenging case---demonstrating how both the strong convexity and smoothness of the costs are exploited---and where $\smash{\tilde{O}(\sqrt{T})}$ rates are possible.

\paragraph{Related work.}

The study of linear control has seen renewed interest in recent years.
Most closely related to our work are
\cite{cohen2018online,agarwal2019online,agarwal2019logarithmic}, that study
online linear control in the full-information setting. The latter paper
establishes logarithmic regret bounds for the case where the costs are strongly
convex and smooth and the noise is i.i.d.~stochastic. Subsequently,
\cite{foster2020logarithmic} established a similar result for fixed and known
quadratic losses and adversarial disturbances.
Thus, we exhibit a gap between the achievable regret
rates in the full- and bandit-feedback cases of our problem. (A similar gap
exists in standard online optimization with strongly convex and smooth losses
\citep{hazan2007logarithmic,shamir2013complexity}.)

A related yet crucially different setting of partial observability was recently
studied in \cite{simchowitz2020improper}, that considered the case where the
state evolution is not revealed in full to the learner and only a
low-dimensional projection of the state vector is observed instead. However,
this model assumes full observability of the (convex) loss function following
each round, and is therefore not directly comparable to ours.
When the underlying linear system is initially unknown (this is the so-called
adaptive control setting), regret of order $\smash{ \sqrt{T} }$ was recently
shown to be optimal for online linear control even with full feedback and
quadratic (strongly convex)
costs~\citep{cassel2020logarithmic,simchowitz2020naive}. Optimal and efficient
algorithms matching these lower bounds were developed earlier
in~\citep{cohen2019learning,mania2019certainty,agarwal2019online}.

In the reinforcement learning literature on finite Markov Decision Processes
(MDPs), regret minimization with bandit feedback was studied extensively (e.g., \cite{neu2010online,jaksch2010near,dekel2013better,dick2014online,azar2017minimax,rosenberg2019online,jin2019learning}). Our study can thus be viewed as a first step in an analogous treatment of bandit learning in continuous linear control.

\section{Preliminaries}

\subsection{Problem Setup: Bandit Linear Control}
\label{sec:setup}

We consider the setting of controlling a known linear dynamical system with
unknown (strongly) convex losses and bandit feedback. The linear system is an
instance of the process 
described in \cref{eq:linDynamics} where $\Astar$ and $\Bstar$ are known,
initialized for simplicity and without loss of generality at $x_0 = 0$.
(Our assumptions on the nature of the various parameters are specified below.)
Our goal is to minimize the total control cost in the following online setting where an \emph{oblivious} adversary chooses cost functions $c_t : \RR[d] \times \RR[k] \to \RR$ for $t \ge 1$. At round $t$:
\begin{enumerate}[nosep,label=(\arabic*)]
    \item The player chooses control $u_t$;
    \item The system transitions to $x_{t+1}$ according to \cref{eq:linDynamics};
    \item The player observes the new state $x_{t+1}$ and the incurred cost $c_t(x_t, u_t)$.
\end{enumerate}
The overall cost incurred is
$
    J(T)
    =
    \sum_{t=1}^{T} c_t(x_t, u_t)
    .
$
We denote by $J_K(T)$ the overall cost of a linear controller $K \in \RR[d \times k]$, which chooses its actions as $u_t = - K x_t$. For such controllers, it is useful to define the notion of strong stability \cite{cohen2018online} (and its refinement due to \cite{agarwal2019logarithmic}), which is essentially a quantitative version of classic stability notions in linear control.
\begin{definition}[strong stability]
    A controller $K$ for the system $(\Astar, \Bstar)$ is $(\kappa, \gamma)-$strongly stable ($\kappa \ge 1$, $0 < \gamma \le 1$) if there exist matrices $Q,L$ such that
    $\Astar +\Bstar K = Q L Q^{-1}$,
    $\norm{L} \le 1 - \gamma$,
    and
    $\norm{K}, \norm{Q},\norm{Q^{-1}} \le \kappa$.
    If additionally $L$ is complex and diagonal and $Q$ is complex, then $K$ is diagonal $(\kappa, \gamma)$-strongly stable.
\end{definition}
For fixed $\kappa, \gamma$, we define regret with respect to the class $\mathcal{K}$ of $(\kappa,\gamma)-$diagonal strongly stable policies
\begin{align}
    \mathcal{K}
    =
    \brk[c]!{
        K \in \RR[d \times k] :\;
        \text{$K$ is $(\kappa,\gamma)$-diagonal strongly stable w.r.t.~$(\Astar, \Bstar)$}
        }.
\end{align}
Beyond its relative simplicity, this class is interesting as it contains an asymptotic global optimum (with respect to \emph{all} policies) when the costs are constrained to a fixed quadratic function, as in classical control. 
The regret compared to $K \in \mathcal{K}$ is given by
$
    R(T, K)
    =
    J(T) - J_K(T)
$.
The pseudo regret is then defined as
\begin{align*}
    \controlPseudoRegret
    =
    \max_{K \in \mathcal{K}}\EEBrk{R_{\mathcal{A}}(T, K)},
\end{align*}
where the expectation is taken with respect to the randomness of the algorithm, and the system noise.

\paragraph{Assumptions.}
Throughout we assume the following.
There are known constants $\kappa_B \ge 1$, and $W, \cG, \cC, \allowbreak \pConvex, \pSmooth, \noiseStd > 0$ such that:
\begin{enumerate}[nosep]
	\item \label{a:1}(System bound) $\norm{\Bstar} \le \kappa_B$;
	\item \label{a:2} (Noise bound) $\norm{\wt} \le W$ $\forall t \ge 1$;
	\item \label{a:3} (Cost bounds) If $\norm{x},\norm{u} \le \cD$ for large enough $\cD$,%
	\footnote{The precise $\cD$ for which this holds will be specified later as an explicit polynomial in the problem parameters.}
	then
	\begin{align*}
		\abs{\ct{x}{u}} \le \cC \cD^2, \;\;\;
		\norm{\nabla_x \ct{x}{u}},
		\norm{\nabla_u \ct{x}{u}} \le \cG \cD;
	\end{align*}
	\item \label{a:4} (Curvature bounds) The costs $\ct{x}{u}$ are $\pConvex$-strongly convex and $\pSmooth$-smooth;
	\item \label{a:5} (Noise) The disturbances $\wt$ are independent random variables satisfying $\EEBrk[0]{\wt \wt\tran} \succeq \noiseStd^2 I$.
\end{enumerate}
The above assumptions are fairly standard in recent literature (e.g., \cite{agarwal2019online,agarwal2019logarithmic}).

\subsection{Online Optimization with Memory}
\label{sec:OOmem}

We describe the setting of online optimization with memory~\cite{arora2012online,anava2015online}, which will serve as an intermediate framework for our algorithmic development.
In this setting, an oblivious adversary chooses loss functions $f_t : \fRange^{\memLen} \to \RR$ over a domain $\fRange \subseteq \RR[d]$, where $H \ge 1$ is the length of the adversary's memory. The game proceeds in rounds, where in round $t$, the player chooses $x_t \in \fRange$ and observes some form of feedback $\bcoFeedback$.
Performance is evaluated using the expected \emph{policy regret},
\begin{equation}
\label{eq:bcoMemRegretDef}
    \bcoMemRegret{T}
    =
    \EEBrk{\sum_{t=\memLen}^{T} f_t(x_{t+1-H},\ldots,x_t)}
    -
    \min_{x \in \fConstraint} \sum_{t=\memLen}^{T} f_t(x, \ldots, x),
\end{equation}
where $\fConstraint \subseteq \fRange$ is the comparator set,
which may differ from the domain $\fRange$ where the loss functions are defined (and are well behaved).
Notice that for $H=1$, the quantity $\bcoMemRegret[1]{T}$ refers to the regret of standard online optimization, with no memory.

We will rely on the following conditions for the loss functions.
The first is a coordinate-wise Lipschitz property, while the second is standard smoothness, stated explicitly for an $\memLen$-coordinate setup.
\begin{definition} \label{def:coordLipschitz}
    $f : \fRange^{\memLen} \to \RR$ is coordinate-wise $\pLip-$Lipschitz if $\forall i \in \brk[s]{\memLen}$, $x_1, \ldots, x_{\memLen}, y_i \in \fRange$:
    \begin{equation*}
        \abs{
            f(x_1, \ldots, x_i, \ldots, x_H)
            -
            f(x_1, \ldots, y_i, \ldots, x_H)
        }
        \le
        \pLip \norm{x_i - y_i}
        .
    \end{equation*}
\end{definition}
\begin{definition} \label{def:fSmooth}
    $f : \fRange^{\memLen} \to \RR$ is $\pSmooth-$smooth if for any
    $
        x=(x_1, \ldots, x_{\memLen}),
        y=(y_1,\ldots,y_{\memLen}) 
        \in
        \fRange^{\memLen}
        :
    $
    \begin{equation*}
        f(y) - f(x)
        \le
        \sum_{i=1}^{\memLen} \nabla_i f(x)\tran (y_i - x_i) + \frac{\pSmooth}{2} \norm{y_i - x_i}^2
        ,
    \end{equation*}
    where $\nabla_i$ is the gradient with respect to $x_i$. When $x_1 = \ldots = x_H = z$, we compress notation to $\nabla_i f(z)$.
\end{definition}

\subsection{Disturbance-Action Policies}
\label{sec:dap}

Online linear control may be approximated by certain loss functions with memory, via a reparameterization suggested in \cite{agarwal2019online,agarwal2019logarithmic} named the Disturbance Action Policy (DAP).
For completeness, we state the parameterization here even though our technical development will be mostly orthogonal.

\begin{definition}[Disturbance-Action Policy]
    For a fixed linear controller $K_0 \in \mathcal{K}$ and parameters $M = (M^{\brk[s]{1}}, \ldots, M^{\brk[s]{H}})$ with $M^{\brk[s]{i}} \in \RR[k \times d]$, a \emph{disturbance-action policy} chooses an action at time $t$ as
    \begin{align*}
        u_t(M)
        =
        - K_0 x_t
        +
        \sum_{i=1}^{H} M^{\brk[s]{i}} w_{t-i},
    \end{align*}
    where for notational convenience we say $w_i = 0$ for $i \le 0$.
\end{definition}
This parameterization reduces the decision of the player at time $t$ to choosing
$
    M_t
    =
    (M_t^{\brk[s]{1}}, \ldots, M_t^{\brk[s]{H}})
    .
$
The comparator set is given by
$
    \MConstraint
    =
    \MiConstraint[1] \times \cdots \times \MiConstraint[\memLen]
    ,
$
where
\begin{align*}
    \MiConstraint
    =
    \brk[c]*{M \in \RR[k \times d] \;:\; \norm{M} \le 2 \kappa_B \kappa^3 (1 - \gamma)^i},
\end{align*}
however, the player may choose $M_t$ from the slightly larger
$
    \MRange
    =
    \brk[c]{2M \;:\; M \in \MConstraint}
    .
$
The following defines the adversary's cost functions, referred to as surrogate or ideal cost functions. It is a summary of Definitions 4.2, 4.4, and 4.5 of \cite{agarwal2019online}, as well as 3.4 of \cite{agarwal2019logarithmic}, and while we do not use it explicitly, we give it here for the sake of concreteness.
\begin{definition} \label{def:surrogateCosts}
    Denote by $M_{0:H}$ a sequence of policies $M_0, \ldots, M_H \in \MRange$. For a controller $K$, let $\tilde{A}_K = \Astar + K \Bstar$, and define:
    \begin{enumerate}[nosep,label=(\arabic*)]
        \item (disturbance-state transfer matrix)
            $
                \Psi_{i}^{K} (M_{0:H-1})
                =
                \tilde{A}_K^i \mathbb{1}_{i \le H}
                +
                \sum_{j=1}^{H} \tilde{A}_K^j \Bstar M_{H-j}^{\brk[s]{i-j}} \mathbb{1}_{i-j \in \brk[s]{1, H}}
                ;
            $
        \item (ideal state)
            $
                y_{t+1}^{K}(M_{0:H-1})
                =
                \sum_{i=0}^{2H} \Psi_{i}^{K}(M_{0:H-1}) \wt[t-i]
                ;
            $
        \item (ideal action)
            $
                \upsilon_{t}^{K}(M_{0:H})
                =
                - K y_{t}^{K}(M_{0:H-1})
                +
                \sum_{i=1}^{H} M_{H}^{\brk[s]{i}} \wt[t-i]
                .
            $
    \end{enumerate}
    The surrogate or ideal costs and their expected value are respectively defined as:
    \begin{align*}
        f_t(M_{0:H})
        =
        \ct{y_t^K(M_{0:H-1})}{\upsilon_t^K(M_{0:H})}
        ,\qquad
        F_t(M_{0:H})
        =
        \EE[w]{\brk[s]{\ct{y_t^K(M_{0:H-1})}{\upsilon_t^K(M_{0:H})}}}
        .
    \end{align*}
\end{definition}
The following are statements of the reduction's key results due to \cite{agarwal2019logarithmic}. Denote:
\begin{align}
\label{eq:def1}
    H
    =
    \gamma^{-1} \log 2\kappa^3 T
    ,\qquad
    \Dxu
    =
    8 \gamma^{-1} \kappa_B \kappa^3 W (H \kappa_B + 1)
    .
\end{align}
The first result relates the costs $f_t, F_t$ and the associated regret to the original losses and regret.
\begin{lemma}%
\label{lemma:controlToOCO}
    For any algorithm $\mathcal{A}$ that plays policies $M_1, \ldots, M_T \in \MRange$ we have:
    \begin{enumerate}[nosep,label=(\roman*)]
        \item $\norm{x_t},\norm{u_t} \le \Dxu$, and thus $\abs{\ct{x_t}{u_t}} \le \cC \Dxu^2$;
        \item 
            $
                \abs{\ct{x_t}{u_t} - f_t(M_{t-H:t})}
                \le
                \cG \Dxu^2 / T
                ;
            $
        \item
            $
                \controlPseudoRegret
                \le
                \EE{}\brk[s]*{
                    \sum_{t=H+1}^{T} F_t(M_{t-H:t})
                    -
                    \min_{M_* \in \MConstraint}
                    \sum_{t=H+1}^{T} F_t(M_*,\ldots,M_*)
                }
                +
                2 \Dxu^2 (\cG + H \cC)
                .
            $
    \end{enumerate}
\end{lemma}

The second result establishes certain desirable properties of the cost functions $f_t$ and $F_t$.

\begin{lemma}%
\label{lemma:SurrogateProperties}
    Let
    $
        \tilde{F}_t : M \mapsto F_t(M,\ldots,M)
        .
    $
    Then:
    \begin{enumerate}[nosep,label=(\roman*)]
        \item $f_t, F_t$ are coordinate-wise $\pLipF$-Lipschitz over $\MRange$, with $\pLipF = 2 \kappa_B \gamma^{-1} \kappa^3 G \Dxu W$;
        \item \label{item:smoothSurrogate} if $\ct{\cdot}{\cdot}$ are $\pSmooth$-smooth then $f_t, F_t, \tilde{F}_t$ are $\pSmoothF$-smooth over $\MRange$ with
        $
            \pSmoothF
            =
            25 \pSmooth \kappa_B^2 \kappa^6 W^2 H / \gamma^2
        $;
        \item If $\ct{\cdot}{\cdot}$ are $\pConvex$-strongly convex and $\EEBrk[0]{\wt\wt\tran} \succeq \noiseStd^2 I$ then $\tilde{F}_t$ are $\pConvexF-$strongly convex over $\MConstraint$ with 
        $
            \pConvexF
            =
            \tfrac{1}{36} \pConvex \noiseStd^2 \gamma^2 / \kappa^{10}
            .
        $
    \end{enumerate}
\end{lemma}
We note that the second claim of \cref{lemma:SurrogateProperties} was not previously established. We prove it in \cref{sec:fSmoothProof} using similar techniques to those used for the first claim.

\section{Algorithm and Main Results}
\label{sec:mainAlg}

We present a new algorithm for the bandit linear control problem, detailed in \cref{alg:main}, for which we prove:
\begin{theorem}
\label{thm:main}
    Let $\memLen, \Dxu, \pConvexF, \pSmoothF$ be as in \cref{eq:def1,lemma:SurrogateProperties}, $K_0$ be a $(\kappa,\gamma)-$diagonal strongly stable controller, and $n = \min\brk[c]{d,k}$.
    Suppose \cref{alg:main} is run with the above parameters, and
    \begin{align*}
        \eta
        =
        \sqrt{\frac{3 n^2 + (15 \pSmoothF / \pConvexF) \log T}{T d^2 k^2 \cC^2 \Dxu^4}}
        ,
        \qquad
        \rti
        =
        \brk[s]*{(2\kappa_B \kappa^3(1-\gamma)^i)^{-2} + \tfrac12 \pConvex_f \eta t}^{-1/2}
        .
    \end{align*}
    Then the expected pseudo-regret is bounded as
    \begin{align*}
        \controlPseudoRegret
        \le
        4 d k \cC \Dxu^2 (\memLen+1)^2
        \sqrt{T \brk{3 n^2 + (15 \pSmoothF / \pConvexF) \log T}}
        +
        \tilde{O}(T^{1/4})
        .
    \end{align*}
\end{theorem}
The big-$\tilde{O}$ notation in the theorem hides polynomial dependence in the problem parameters and poly-log dependence on $T$. We prove \cref{thm:main} later in \cref{sec:mainAlgProof} after discussing our reduction technique.

There are three main components to the algorithm:
\begin{enumerate}[leftmargin=!,label=(\arabic*)]
    \item A randomized schedule to determine the times of parameter updates, which ensures that these are at least $2(H+1)$ apart and $O(H)$ in expectation. This is part of our new reduction scheme, which is presented and discussed in \cref{sec:bcoMem}.%
    \item A standard one-point gradient estimate that gives a (nearly-)unbiased estimate for the gradient of a function based on bandit feedback, by perturbing the policy using uniform samples from the unit Euclidean sphere of $\RR[(d \times k) \times H]$; this is denoted as 
    $U \sim \sphere{(k \times d) \times \memLen}$.
    \item A preconditioned (online) gradient update rule that uses mixed regularization and standard Euclidean projections
    $
        \Pi_{\MiConstraint}\brk[s]{M}
        =
        \argmin_{M' \in \MiConstraint} \norm{M - M'}_F
        .
    $
\end{enumerate}

The mixed regularization, inspired by \cite{hazan2014bandit}, is comprised of two terms (see $\smash{ \rti }$ in \cref{thm:main}):
the first exploits the strong convexity of the (expected) loss functions, while the second accounts for the small diameter of~$\MiConstraint$, which might be significantly smaller than the magnitude of the perturbations required for the gradient estimates (this is particularly problematic for large $i$).
To avoid sampling too far away from the ``admissable'' set, where the cost functions are well-behaved, we cap the perturbations of the one-point estimate according to the radii $\smash{ \brk[c]{\rti}_{i \in \brk[s]{\memLen}} }$ and increase the regularization term to account for the higher variance of the resulting gradient estimate.

\begin{algorithm}[t]
	\caption{Bandit Linear Control} \label{alg:main}
	\begin{algorithmic}[1]
		\State \textbf{input:}
		    controller $K_0$,
		    memory length $\memLen$,
	        step size $\eta$,
	        and coordinate-wise sampling radii $\rti$
        \State 
            \textbf{Draw}
		        $\Ut[1] \sim \sphere{(k \times d) \times \memLen}$.
	    \State \textbf{Initialize}
	        $\tau = 1$,
		    $\barMt[1] = 0$,
		    $\Mti[1] = \rti[1]\Uti[1] \quad (\forall i \in \brk[s]{\memLen}) .$
	        
        \For{$t = 1, \ldots, T$}
            \State \textbf{Play}
                $
                    \ut
                    =
                    - K_0 \xt
                    +
                    \sum_{i=1}^{\memLen} \Mti \wt[t-i],
                $
            \State \textbf{Observe}
                $\xt[t+1]$ and $\ct{\xt}{\ut}$;
                update
                $\wt = \xt[t+1] - \Astar \xt - \Bstar \ut$.
            \State \textbf{Draw} $\bt \sim $ Bernoulli($1/(2\memLen+2)$)
            \If{
                $t \ge 2\memLen+2$
                and
                $\bt \prod_{i=1}^{2\memLen + 1} (1-\bt[t-i]) = 1$
                }
                \State \textbf{Update}
                    $
                        \barMti[\tau+1]
                        =
                        \Pi_{\MiConstraint}\brk[s]!{\barMti - \eta  d k \memLen \ct{\xt}{\ut} \rti[\tau]\Uti}
                    $
                    \quad
                    $(\forall i \in \brk[s]{\memLen}).$
                \State
                    $\tau \gets \tau + 1$
                \State
                    \textbf{Draw}
        		        $\Ut \sim \sphere{(k \times d) \times \memLen}$.
                \State 
                    $\Mti[t+1]
                    =
                    \barMti + \rti[\tau]\Uti$
                    \quad
                    $(\forall i \in \brk[s]{\memLen}).$
            \Else
                \State $\Mt[t+1] = \Mt$
            \EndIf
        \EndFor
	\end{algorithmic}
\end{algorithm}

\section{Bandit Convex Optimization with Memory}
\label{sec:bcoMem}

In this section we give the details of our reduction from BCO with memory to standard BCO, that constitutes a key element of our main algorithm.
The application to bandit linear control, however, will require a slightly more general feedback model than the usual notion of bandit feedback. %

\subsection{Setup}

We consider the online optimization with memory setting described in \cref{sec:OOmem}, with feedback model such that on round $t$:
\begin{enumerate}[nosep,label=(\arabic*),leftmargin=!]
    \item The player chooses $x_t \in \fRange$, and independently, the adversary draws a random $\xi_t$;
    \item The player observes feedback
    $
        \bcoFeedback
        =
        \bcoFeedback(x_{t+1-\memLen:t}; \xi_{t+1-\memLen:t})
    $
    such that, if $x_{t+1-\memLen:t}$ are jointly independent of $\xi_{t+1-\memLen:t}$, then
    $%
        \abs{\EE[\xi_{t+1-\memLen:t}]{\brk[s]{\bcoFeedback}} - f_t(x_{t+1-\memLen:t})}
        \le
        \varepsilon.
    $%
\end{enumerate}
The above expectation is only with respect to the variables $\xi_{t+1-\memLen:t}$, and $\varepsilon \ge 0$ is a fixed parameter (possibly unknown to the player).
Our feedback model, which may potentially seem non-standard, encompasses the following ideas, both of which are necessary for the application to linear control: 
\begin{itemize}[nosep,leftmargin=!]
    \item In the standard no-memory setting ($\memLen = 1$), standard arguments apply even if the feedback received by the learner is randomized, as long as it is independent of the learner's decision on the same round.
    In the memory setting, the analogous condition is that the last $\memLen$ decisions do not depend on the adversary's randomness during this time. 
    \item We allow feedback of the form $\bcoFeedback = f_t(x_t) + \varepsilon_t$, where $\varepsilon_t$ is a small \emph{adaptive} adversarial disturbance that can depend on all past history
    (yet is at most $\varepsilon$ in absolute value).
\end{itemize}

\subsection{Base BCO Algorithm}
\label{sec:baseBCO}

The reduction relies on the following properties of the base algorithm $\baseBCOof{T}$ for BCO with no memory, that can be used against an adversary that chooses loss function from $\funcSet \subseteq \brk[c]{f : \fRange \to \RR}$, and observing feedback satisfying
$
    \abs{\EE[\xi_t]{\brk[s]{\bcoFeedback[t]}} - f_t(x_t)}
    \le
    \varepsilon
    :
$
\begin{enumerate}[nosep,leftmargin=!,label=(\roman*)]
    \item Its regret at times $t \le T$ is bounded as $\bcoMemRegret[1]{t} \le \algRegretBound{T}$ where $\algRegretBound{T} \ge 0$;
    \item Its predictions $x_1, \ldots, x_T$ satisfy
        $
            \norm{\bar{x}_{t+1} - \bar{x}_t}
            \le
            \pMeanShift
        $
        , and
        $
            \norm{x_{t} - \bar{x}_t}
            \le
            \pPredVar
        $
        almost surely, where $\bar{x}_t$ is the expected value of $x_t$ conditioned on all past history up to (not including) the player's decision at time $t$, and $\pMeanShift,\pPredVar$ are decreasing sequences.
\end{enumerate}
The above properties are satisfied by standard BCO algorithms, often without any modification. 
In particular, these algorithms are amenable to our randomized and perturbed feedback model and often require only a slight modification in their regret analyses to account for the additive disturbances.
(In \cref{sec:baseBCOproofs} we give an analysis of a concrete BCO algorithm in this setting.)

Notice that, crucially, $\pMeanShift$ bounds the change in the algorithm's \emph{expected} predictions as opposed to their actual movement. This is crucial as typical BCO algorithms add large perturbations to their predictions (as part of their gradient estimation procedure), with magnitude often significantly larger than the change in the underlying expected prediction; i.e., it holds that $\pPredVar \gg \pMeanShift$.
Our reduction procedure is able to exploit this observation to improve performance for smooth functions.

\subsection{The Reduction}

We can now present our reduction from BCO with memory to BCO with no memory ($\memLen = 1$); see \cref{alg:bcoReduction}.
The idea is simple: we use a base algorithm for standard BCO, denoted here by $\baseBCO$, using the observed feedback, but make sure that $\baseBCO$ is updated at most once in every $H$ rounds. Since the setup is adversarial, we cannot impose a deterministic update schedule; instead, we employ a randomized schedule in which $\baseBCO$ is invoked with probability $1/H$ on each round,
but constrained~so that this does not happen too frequently.
(A similar technique was used in a different context in~\cite{dekel2014blinded,cesa2018nonstochastic}.)

The induced spacing between consecutive updates of $\baseBCO$ serves two purposes at once: first, it reduces the $H$-memory loss functions to functions of a single argument, amenable to optimization using $\baseBCO$; second, it facilitates (conditional) probabilistic independence between consecutive updates which is crucial for dealing with the extended feedback model as required by the application to linear control.
(We note that these conditions are not satisfied by existing techniques~\cite{arora2012online,anava2015online,agarwal2019online}.)

\begin{algorithm}[ht]
	\caption{BCO Reduction} \label{alg:bcoReduction}
	\begin{algorithmic}[1]
		\State \textbf{input:} memory length $\memLen$, BCO algorithm $\baseBCOof{T / \memLen}$.
		\State \textbf{set:} $x_1 \gets$ $\mathcal{A}$.initialize()
        
        \For{$t = 1, \ldots, T$}
            \State Play $x_t$ (independently, adversary draws $\xi_t$)
            \State Observe feedback $\bcoFeedback(x_{t+1-\memLen:t}; \xi_{t+1-\memLen:t})$
            \State Draw $\bt \sim $ Bernoulli($1 / \memLen$)
            \If{$t \ge \memLen$ and $\bt \prod_{i=1}^{\memLen - 1} (1-\bt[t-i]) = 1$}
                \State $x_{t+1} \gets$ $\mathcal{A}$.update($\bcoFeedback$)
            \Else
                \State $x_{t+1} \gets x_t$
            \EndIf
        \EndFor
	\end{algorithmic}
\end{algorithm}

The following is the main result of the reduction.
\begin{theorem} \label{thm:BCOreduction}
    Suppose \cref{alg:bcoReduction} is run using $\mathcal{A}(T / \memLen)$ satisfying the above properties:
    \begin{enumerate}[nosep,label=(\roman*)]
        \item  If $\tilde{f}_t : x \mapsto f_t(x, \ldots, x)$ satisfy $\tilde{f}_t \in \funcSet$, and $f_t$ are coordinate-wise $\pLip-$Lipschitz then
        \begin{align*}
            \bcoMemRegret{T}
            \le
            3 \memLen \algRegretBound{\frac{T}{\memLen}}
            +
            \frac12 \pLip \memLen^2 \sum_{t = 1}^{\floor{T / \memLen}} \brk{ \pMeanShift + 2\pPredVar };
        \end{align*}
        \item If additionally $\tilde{f}_t$ are convex and $f_t$ are $\pSmooth-$smooth, then
        \begin{align*}
            \bcoMemRegret{T}
            \le
            3 \memLen \algRegretBound{\frac{T}{\memLen}}
            +
            \frac12 \memLen^2 \sum_{t=1}^{\floor{T / \memLen}+1} \brk{\pLip \pMeanShift + \pSmooth \pMeanShift^2 + 6 \pSmooth \pPredVar^2}.
        \end{align*}
    \end{enumerate}
\end{theorem}

\subsection{Proof Ideas}

We provide some of the main ideas for proving \cref{thm:BCOreduction}. We start with the following technical lemma that quantifies the duration between updates of the base BCO algorithm.
\begin{lemma} \label{lemma:intervalLengths}
    Suppose $\bt$ in \cref{alg:bcoReduction} are drawn in advance for all $t \ge 1$. Let $t_0 = 0$ and for $i \ge 1$ let
    \begin{equation*}\textstyle
        t_i
        =
        \min\brk[c]*{t \ge t_{i-1} + \memLen \;\mid\; \bt \prod_{i=1}^{H-1} (1-\bt[t-i]) = 1}.
    \end{equation*}
    Then denoting $S = \brk[c]*{t_i \;|\; \memLen \le t_i < T}$, the times \cref{alg:bcoReduction} updates $\mathcal{A}$, we have that (i) $\abs{S} \le \floor{T / H},$ and (ii) $\EEBrk{t_i - t_{i-1}} = \EE{t_1} \le 3H$ for all $i$.
\end{lemma}
See proof in \cref{sec:proofsOfBCOMemLemmas}.
The next lemma considers \cref{alg:bcoReduction} as a BCO without memory algorithm that incurs loss $\tilde{f}_t(x_{t+1-\memLen})$ at each round and relates its regret to that of the base algorithm.
\begin{lemma} \label{lemma:BCOmemNoMem}
    Suppose \cref{alg:bcoReduction} is run with $\baseBCOof{T / \memLen}$ as in \cref{thm:BCOreduction}. If $\tilde{f}_t \in \funcSet$ then we have that
    \begin{align*}
        \EEBrk{\sum_{t=\memLen}^{T} \tilde{f}_t(x_{t+1-\memLen})} - \sum_{t=\memLen}^{T} \tilde{f}_t(x)
        \le
        3\memLen \algRegretBound{\frac{T}{\memLen}}
        ,\;\; \forall x \in \fConstraint.
    \end{align*}
\end{lemma}
\begin{proof}
    Let $S$ be the times \cref{alg:bcoReduction} updates $\mathcal{A}$ as defined in \cref{lemma:intervalLengths}.
    Denote $\tilde{\xi}_t = \xi_{t+1-\memLen:t}$, and notice that $\brk[c]{\tilde{\xi}_t}_{t \in S}$ are mutually independent since \cref{alg:bcoReduction} ensures there are at least $\memLen$ rounds between updates of $\baseBCO$. Moreover, this implies that for any $t \in S$, $x_{t+1-\memLen} = \ldots = x_t$, and these are also independent of $\tilde{\xi}_{t}$. Our $\memLen$-memory feedback model thus implies that
    \begin{align*}
        \abs{\tilde{f}_t(x_t) - \EE[\tilde{\xi}_{t}]{\brk[s]{\bcoFeedback}}}
        \le
        \varepsilon
        , \qquad
        \forall t \in S,
    \end{align*}
    and since $\tilde{f}_t \in \funcSet$, we can use the regret bound of $\baseBCO$ to get that for any $x \in \fConstraint$
    \begin{equation*}
        \EEBrk{\sum_{t \in S} \tilde{f}_t(x_t) - \sum_{t \in S} \tilde{f}_t(x)}
        =
        \EEBrk{\EEBrk{\sum_{t \in S} \tilde{f}_t(x_t) - \sum_{t \in S} \tilde{f}_t(x) \;\bigg|\; S}}
        =
        \EEBrk{\bcoMemRegret[1]{\abs{S}}}
        \le
        \algRegretBound{\frac{T}{H}}
        ,
    \end{equation*}
    where the last transition also used the fact that $\abs{S} \le T / \memLen$ (see \cref{lemma:intervalLengths}).
    Next, denote $\isUpRound = \bt \prod_{i=1}^{\memLen - 1} (1-\bt[t-i])$ and notice that $\EE{\isUpRound} = \EE{\isUpRound[\memLen]}$.
    Then for any fixed $x \in \fConstraint$ we have that
    \begin{align*}
        \EEBrk{\sum_{t \in S} \tilde{f}_t(x)}
        =
        \EEBrk{\sum_{t=\memLen}^{T} \tilde{f}_t(x) \isUpRound}
        =
        \sum_{t=\memLen}^{T} \tilde{f}_t(x) \EEBrk{\isUpRound}
        =
        \EEBrk{\isUpRound[\memLen]} \sum_{t=H}^{T} \tilde{f}_t(x).
    \end{align*}
    Next, notice that $x_{t+1-\memLen}$ is independent of $\isUpRound$ and since $\isUpRound = 1$ implies that $x_t = x_{t+1-\memLen}$, we get that
    \begin{align*}
        \EEBrk{\sum_{t \in S} \tilde{f}_t(x_t)}
        =
        \EEBrk{\sum_{t=\memLen}^{T} \tilde{f}_t(x_t) \isUpRound}
        =
        \sum_{t=\memLen}^{T} \EEBrk{\tilde{f}_t(x_{t+1-\memLen})} \EEBrk{\isUpRound}
        =
        \EEBrk{\isUpRound[\memLen]} \EEBrk{\sum_{t=\memLen}^{T} \tilde{f}_t(x_{t+1-\memLen})}.
    \end{align*}
    Finally, combining the last three equations we get that
    \begin{align*}
        \EEBrk{\sum_{t=\memLen}^{T} \tilde{f}_t(x_{t+1-\memLen})} - \sum_{t=\memLen}^{T} \tilde{f}_t(x)
        \le
        \brk{\EEBrk{\isUpRound[\memLen]}}^{-1} \algRegretBound{\frac{T}{\memLen}}
        \le
        3\memLen \algRegretBound{\frac{T}{\memLen}},
    \end{align*}
    where the last transition used the non-negativity of $\algRegretBound{T / \memLen}$ and that $\EEBrk{\isUpRound[H]} \ge 1 / 3 \memLen$.
\end{proof}

Given the above, completing the proof of \cref{thm:BCOreduction} entails bounding 
$
    \EE{}\big[\sum_{t=\memLen}^{T} f_t(x_{t+1-\memLen}, \ldots, x_t) - \tilde{f}_t(x_{t+1-\memLen})\big]
    .
$
While the smooth case requires some delicate care for achieving the squared dependence on $\pPredVar$, the proof is otherwise quite technical and thus deferred to \cref{sec:reductionProof}.

\section{Analysis}
\label{sec:mainAlgProof}

We first require the following, mostly standard, analysis of the base procedure of \cref{alg:main} for the no-memory setting ($\memLen=1$). See proof in \cref{sec:proofOfBaseBCOlemma}.
\begin{lemma} \label{lemma:baseBCO}
    Consider the setting of \cref{sec:bcoMem} with $H=1$ and $\varepsilon \in \tilde{O}(1/T)$, against an adversary that chooses $f_t: \MRange \to \RR$ that are $\pConvexF$-strongly convex and $\pSmoothF$-smooth. Let $\dM = d k \memLen$ be the dimension of $\MConstraint$, and $\bcoMemRegret[1]{t}$ be the regret of a procedure that at time $t$:
    \begin{enumerate}[nosep,label=(\roman*)]
        \item Draws $\Ut[t] \sim \sphere{(k \times d) \times \memLen}$; and plays $\Mt$ where $\Mti[t] = \barMti[t] + \rti[t]\Uti[t]$ $(\forall i \in \brk[s]{\memLen})$
        \item Observes $\bcoFeedback$; and sets $\gHat^{\brk[s]{i}} = (d_{\MConstraint} / \rti[t]) \bcoFeedback \Uti[t]$ $(\forall i \in \brk[s]{\memLen})$
        \hfill (1-point gradient estimate)
        \item Updates
            $
                \barMti[t+1] 
                =
                \Pi_{\MiConstraint}\brk[s]!{\barMti[t] - \eta  \brk{\rti[t]}^2 \gHat^{\brk[s]{i}}}
                \quad
                (\forall i \in \brk[s]{\memLen})
                .
            $
            \hfill (preconditioned update)
    \end{enumerate}
    If $\abs{\bcoFeedback} \le \bcoFeedbackBound$, $\rti$, $n$ are as in \cref{thm:main}, and $\eta \in \tilde{O}(T^{-1/2})$ then
    $\pMeanShift = d_{\MConstraint} \bcoFeedbackBound \sqrt{2 \eta / \pConvexF t}$
    and
    $\pPredVar^2 = 2 / \pConvexF \eta t$
    satisfy the assumptions of $\baseBCOof{T}$ in \cref{thm:BCOreduction}
    ,
    and for all $t \le T$
    \begin{align*}
        \bcoMemRegret[1]{t}
        \le
        \frac{1}{\eta} \brk*{\memLen n^2 + \frac{2\pSmoothF}{\pConvexF}(1+\log T)}
        +
        \frac{d_{\MConstraint}^2 \bcoFeedbackBound^2}{2} \eta T
        +
        \tilde{O}\brk{T^{1/4}}
        .
    \end{align*}
\end{lemma}

\begin{proof}[of \cref{thm:main}]

Consider $f_t, F_t$ from \cref{def:surrogateCosts} and notice that $f_t$ depends on the last $\memLen+1$ policies but also the last $2(\memLen+1)$ system noises. This means that the effective memory of the adversary is $2(\memLen+1)$, prompting us to modify the definitions of $f_t,F_t$ to receive $2(\memLen+1)$ policies but ignore the first $\memLen+1$, i.e.,
\begin{align*}
    F_t^{\text{new}}(M_{0:2\memLen+1}) = F_t(M_{\memLen+1:2\memLen+1})
    ,\qquad
    f_t^{\text{new}}(M_{0:2\memLen+1}) = f_t(M_{\memLen+1:2\memLen+1})
    .
\end{align*}
Henceforth, $f_t, F_t$ refer to $f_t^{\text{new}}, F_t^{\text{new}}$. Notice that \cref{lemma:controlToOCO,lemma:SurrogateProperties} are not impacted by this change and hold with the same $\memLen$ as the original functions. We thus have that $F_t$ are $\pLipF$-coordinate-wise Lipschitz and $\pSmoothF$-smooth, and $\tilde{F}_t$ are $\pConvexF$-strongly convex and $\pSmoothF$-smooth. Moreover, $\abs{\ct{x_t}{u_t}} \le \cC \Dxu^2$, and if $M_{t- 1 -2\memLen:t}$ are independent of $\wt[t-1-2\memLen:t]$ then
\begin{align*}
    \abs{
        \EE[{\wt[t-1-2\memLen:t]}]{\brk[s]{\ct{x_t}{u_t} - f_t(M_{t-1-2\memLen:t})}}
    }
    =
    \abs{
        \EE[{\wt[t-1-2\memLen:t]}]{\brk[s]{\ct{x_t}{u_t}} - F_t(M_{t-1-2\memLen:t})}
    }
    \le
    \frac{\cG \Dxu^2}{T}.
\end{align*}
Now, Consider \cref{alg:main} in the context of the BCO with memory setting presented in \cref{sec:bcoMem} with $\varepsilon = {\cG \Dxu^2}/{T}$, feedback bounded by $\cC \Dxu^2$, and let $\bcoMemRegret[2(\memLen+1)]{T}$ be its regret against an adversary that chooses functions $h_t : \MRange^{2(\memLen+1)} \to \RR$ satisfying:
\begin{itemize}[nosep]
    \item $h_t$ are $\pLipF$-coordinate-wise Lipschitz and $\pSmoothF$-smooth;
    \item $\tilde{h}_t : M \mapsto h_t(M,\ldots,M)$ are $\pConvexF$-strongly convex and $\pSmoothF$-smooth.
\end{itemize}
Since $F_t$ satisfy these assumptions, and our choice of $\rti$ ensures that $M_t \in \MRange$, \cref{lemma:controlToOCO} yields that
\begin{align*}
    \controlPseudoRegret
    \le
    \bcoMemRegret[2(\memLen+1)]{T}
    +
    2 \Dxu^2 (\cG + H \cC)
    ,
\end{align*}
and since the second term is at most poly-log in $T$, it remains to bound $\bcoMemRegret[2(\memLen+1)]{T}$.
To that end, notice that \cref{alg:main} fits the mold of our reduction procedure given in \cref{alg:bcoReduction} with base procedure as in \cref{lemma:baseBCO}. Now, invoking \cref{lemma:baseBCO} with $\bcoFeedbackBound = \cC \Dxu^2$ and horizon $T / 2(\memLen+1)$, the second term of \cref{thm:BCOreduction} satisfies that
\begin{align*}
    \sum_{t=1}^{\floor{T / 2(\memLen+1)}+1} \brk{\pLipF \pMeanShift + \pSmoothF \pMeanShift^2 + 6 \pSmoothF \pPredVar^2}
    \le
    \frac{12 \pSmoothF \log T}{\pConvexF \eta}
    +
    \tilde{O}(T^{1/4})
    ,
\end{align*}
and further using \cref{lemma:baseBCO} to bound the first term of \cref{thm:BCOreduction}, and simplifying, we get that
\begin{align*}
    \bcoMemRegret[2(\memLen+1)]{T}
    \le
    2 (\memLen+1)^2 \brk[s]*{
        \frac{1}{\eta} \brk*{3 n^2 + \frac{15\pSmoothF}{\pConvexF}\log T}
        +
        d^2 k^2 \cC^2 \Dxu^4 \eta T
    }
    +
    \tilde{O}(T^{1/4})
    .
\end{align*}
Our choice of $\eta$ yields the final bound.
\end{proof}

\section{Extensions to General Costs and Adversarial Noise}
\label{sec:extension}
In this section we consider the case where the cost functions chosen by the adversary are general, possibly non-smooth (weakly) convex functions. Importantly, we also allow the system noise to be chosen by an oblivious adversary. Formally, the setup is identical to \cref{sec:setup} but with the following modifications:
\begin{enumerate}
    \item Only Assumptions \ref{a:1}-\ref{a:3} are assumed throughout;
    \item The costs $\ct{x}{u}$ are (weakly) convex functions of $(x,u)$;
    \item The disturbances $\wt$ are chosen by an oblivious adversary, i.e., one that has knowledge of the algorithm but must choose all disturbances before the first round.
    (Notice that $\wt$ are still bounded as per Assumption \ref{a:1}.)
\end{enumerate}
To ease notation, recall that $n = \min\brk[c]{d,k}$, and
$
    \constraintD^2
    =
    \max_{M_1, M_2 \in \MConstraint} \norm{M_1 - M_2}_F^2
    \le
    4 n^2 \kappa_B^2 \kappa^6 / \gamma
    .
$
The following extends our main results, given in \cref{thm:main}, to the setting above.

\begin{theorem}
\label{thm:mainExtended}
    Let $\memLen, \Dxu, \pLipF, \pSmoothF$ be as in \cref{eq:def1,lemma:SurrogateProperties}, $K_0$ be a $(\kappa,\gamma)$-strongly stable controller, $\dM = d k \memLen$, $\bcoFeedbackBound = \cC \Dxu^2$, and 
    $
    \rtOf[0]{i}
    =
    2\kappa_B \kappa^3(1-\gamma)^i
    $.
    Then:
    \begin{enumerate}
        \item The regret of running \cref{alg:main} with 
        $
            \eta 
            =
            2\brk[s]*{\frac{(\memLen+1)^3 \pLipF^2 \constraintD^2}{\dM^6 \bcoFeedbackBound^6 T}}^{1/4}
        $,
        and
        $
            \rti 
            =
            \brk[s]*{
            (\rti[0])^{-2}
            +
            \frac{4 \pLipF \sqrt{(\memLen+1)T}}{\dM \bcoFeedbackBound \constraintD}
            }^{-1/2}
        $
        satisfies 
        \begin{align*}
            \controlRegret
            \le
            13 \sqrt{2 d k n \cC \Dxu^2 \kappa_B \kappa^3 \gamma^{-1/2} \pLipF (\memLen+1)^{7/2}} T^{3/4}
            +
            \tilde{O}(T^{1/2})
            ;
        \end{align*}
        \item if $c_t$ are $\pSmooth$-smooth, then the regret of running \cref{alg:main} with
        $
            \rti 
            =
            \brk[s]*{
            (\rti[0])^{-2} 
            +
            \brk{\frac{4 \pSmoothF^2 T}{(\memLen+1)\dM^2 \bcoFeedbackBound^2 \constraintD^2}}^{1/3}
            }^{-1/2}
        $,
        and
        $
            \eta
            =
            \brk[s]*{
            \frac{2 (\memLen+1) \pSmoothF \constraintD^2}{\dM^4 \bcoFeedbackBound^4 T}
            }^{1/3}
        $
        satisfies
        \begin{align*}
            \controlRegret
            \le
            12 \brk*{2 d k n \cC \Dxu^2 \kappa_B \kappa^3 \sqrt{\pSmoothF / \gamma} (\memLen+1)^3 T}^{2/3}
            +
            \tilde{O}(T^{1/2})
            .
        \end{align*}
    \end{enumerate}
\end{theorem}
The proof of \cref{thm:mainExtended} is given in \cref{sec:extensionAnalysis}, and follows the same ideas behind \cref{thm:main} with a few technical adjustments. Notice that \cref{thm:mainExtended} only requires a strongly stable initial controller $K_0$, as opposed to the \emph{diagonal} strongly stable controller needed for \cref{thm:main}. Moreover, since the noise is no longer stochastic, our definition of pseudo regret now coincides with the standard definition of regret given by
$
    \controlRegret
    =
    \EEBrk[!]{\max_{K \in \mathcal{K}} R_{\mathcal{A}}(T, K)}
    .
$

\begin{ack}
This work was partially supported by the Israeli Science Foundation (ISF) grant 2549/19 and by the Yandex Initiative in Machine Learning. 
\end{ack}

\bibliography{bibliography}

\newpage
\appendix
\allowdisplaybreaks

\section{Reduction to no-Memory BCO Proofs}

\subsection{Proof of \cref{thm:BCOreduction}}
\label{sec:reductionProof}
We first need the following lemma, which bounds the prediction shifts and magnitudes of \cref{alg:bcoReduction}.
\begin{lemma} \label{lemma:stablityExpressions}
    Suppose \cref{alg:bcoReduction} is run with $\baseBCOof{T / \memLen}$ as in \cref{thm:BCOreduction}, and let $x_1, \ldots, x_T$ be its predictions. Then we have that for $q > 0$:
    \begin{enumerate}[nosep,label=(\roman*)]
        \item
        $
            \sum_{t=\memLen}^{T} \sum_{i=2}^{\memLen} \norm{x_{t+i-\memLen} - x_{t+1-\memLen}}
            \le
            \sum_{t = 1}^{\floor{T / \memLen}} \brk{\pMeanShift + 2\pPredVar}
            ;
        $
        \item
        $
            \sum_{t=\memLen}^{T} \sum_{i=1}^{\memLen}
                \norm{\bar{x}_{t+i-\memLen} - \bar{x}_{t+1-\memLen}}^q
            \le
            \frac12 \memLen^2 \sum_{t=1}^{\floor{T / \memLen}} \pMeanShift^q
            ;
        $
        \item
        $
            \EEBrk[!]{\sum_{t=\memLen}^{T} \sum_{i=1}^{\memLen}
                \norm{x_{t+i-\memLen} - \bar{x}_{t+i-\memLen}}^2}
            \le 3 \memLen^2 \sum_{t=1}^{\floor{T / \memLen}+1} \pPredVar^2
            .
        $
    \end{enumerate}
\end{lemma}
The proof is mostly technical, and relies on the fact that \cref{alg:bcoReduction} changes its prediction at most once every $\memLen$ rounds. See proof in \cref{sec:proofsOfBCOMemLemmas}. We are now ready to prove \cref{thm:BCOreduction}.

\begin{proof}[of \cref{thm:BCOreduction}]
    We show that \cref{alg:bcoReduction} achieves the desired regret bound. Given \cref{lemma:BCOmemNoMem}, the proof is concluded by upper bounding
    $
        \EEBrk[!]{\sum_{t=\memLen}^{T} f_t(x_{t+1-\memLen},\ldots, x_t) - \tilde{f}_t(x_{t+1-\memLen})}
    $
    under each set of assumptions. First, using the coordinate-wise Lipschitz property we get that
    \begin{align*}
        \sum_{t=\memLen}^{T} f_t(x_{t+1-\memLen},\ldots, x_t) - \tilde{f}_t(x_{t+1-\memLen})
        \le
        \pLip \sum_{t=\memLen}^{T} \sum_{i=2}^{\memLen} \norm{x_{t+i-\memLen} - x_{t+1-\memLen}}
        \le
        \frac12 \pLip \memLen^2 \sum_{t = 1}^{\floor{T / \memLen}} \brk{ \pMeanShift + 2\pPredVar },
    \end{align*}
    where the last transition follows by \cref{lemma:stablityExpressions}. Taking expectation concludes the first part of the proof.
    Now, notice that by its definition, $\bar{x}_t$ is determined given any history up to (not including) the player's decision at time $s \ge t$. Using total expectation we thus get that
    $%
        \EEBrk[0]{\nabla_i f_t(\bar{x}_{t})\tran x_{s}}
        =
        \EEBrk[0]{\nabla_i f_t(\bar{x}_{t})\tran \bar{x}_{s}}
        ,
    $ %
    and using this equality we get that for all $i \ge 1$,
    \begin{align*}
        \EEBrk{\nabla_i f_t(\bar{x}_{t+1-\memLen})\tran (x_{t+i-\memLen} - \bar{x}_{t+1-\memLen})}
        &=
        \EEBrk{\nabla_i f_t(\bar{x}_{t+1-\memLen})\tran (\bar{x}_{t+i-\memLen} - \bar{x}_{t+1-\memLen})} \\
        \tag{Cauchy-Schwarz}
        &\le
        \EEBrk{\norm{\nabla_i f_t(\bar{x}_{t+1-\memLen})} \norm{\bar{x}_{t+i-\memLen} - \bar{x}_{t+1-\memLen}}} \\
        \tag{$f_t$ Lipschitz}
        &\le
        \pLip \, \EEBrk{\norm{\bar{x}_{t+i-\memLen} - \bar{x}_{t+1-\memLen}}},
    \end{align*}
    where the last transition used the Lipschitz assumption to bound the gradient. Finally, we get that
    \begin{align*}
        &\EEBrk{\sum_{t=\memLen}^{T} f_t(x_{t+1-\memLen},\ldots, x_t) - \tilde{f}_t(x_{t+1-\memLen})} \\
        \tag{$\tilde{f}_t$ convex}
        &\le
        \EEBrk{\sum_{t=\memLen}^{T} f_t(x_{t+1-\memLen},\ldots, x_t) - \tilde{f}_t(\bar{x}_{t+1-\memLen})} \\
        \tag{$f_t$ smooth}
        &\le
        \EEBrk{\sum_{t=\memLen}^{T} \sum_{i=1}^{\memLen}
            \nabla_i f_t(\bar{x}_{t+1-\memLen})\tran (x_{t+i-\memLen} - \bar{x}_{t+1-\memLen})
            +
            \frac{\pSmooth}{2}\norm{x_{t+i-\memLen} - \bar{x}_{t+1-\memLen}}^2} \\
        &\le
        \EEBrk{\sum_{t=\memLen}^{T} \sum_{i=1}^{\memLen}
            \pLip \norm{\bar{x}_{t+i-\memLen} - \bar{x}_{t+1-\memLen}}
            +
            \pSmooth \norm{\bar{x}_{t+i-\memLen} - \bar{x}_{t+1-\memLen}}^2
            +
            \pSmooth \norm{x_{t+i-\memLen} - \bar{x}_{t+i-\memLen}}^2
            }
        ,
    \end{align*}
    where the last transition used the previous equation and the triangle inequality. Plugging in the expressions provided in \cref{lemma:stablityExpressions} concludes the proof.
\end{proof}

\subsection{Proofs of \cref{lemma:intervalLengths,lemma:stablityExpressions}}
\label{sec:proofsOfBCOMemLemmas}
\begin{proof}[of \cref{lemma:intervalLengths}]
    First, notice that by definition $t_i - t_{i-1} \ge H$. Summing over $i$ and recalling $t_0 = 0$ we get that $t_i \ge iH$. By definition of $S$ we then get that $\abs{S}H \le t_{\abs{S}} < T$, and changing sides concludes the first part of the lemma.

    To see the second part of the lemma,
    consider a Markov chain with states corresponding to $H$-tuples of bits that captures the evolution of the sequence $(b_{t-H+1},\ldots,b_t)$ as $t$ increases. Notice that our quantity of interest is the expected return time of the state $s = (0,0,\ldots,1)$. 
    Since the chain is irreducible it admits a stationary distribution $\pi^*$, and by a standard fact about Markov chains (e.g., Proposition 1.14 in \cite{levin2017markov}), the desired expected return time equals $1/\pi^*(s)$. The latter probability equals $(1/H) (1-1/H)^{H-1} \geq 1/(eH) \geq 1/(3H),$ which gives the claim.
\end{proof}

\begin{proof}[of \cref{lemma:stablityExpressions}]
    For the first claim, noticing that the algorithm only changes predictions at times $t \in S$, we have that
    \begin{align*}
        \sum_{t=\memLen}^{T} \sum_{i=2}^{\memLen} \norm{x_{t+i-\memLen} - x_{t+1-\memLen}}
        &\le
        \tag{triangle in.Eq}
        \sum_{i=2}^{\memLen} \sum_{j=1}^{i - 1} \sum_{t=\memLen}^{T} \norm{x_{t+j+1-\memLen} - x_{t+j-\memLen}} \\
        &\le
        \frac12 \memLen^2 \sum_{t=1}^{T-1} \norm{x_{t+1} - x_{t}} \\
        &=
        \frac12 \memLen^2 \sum_{t \in S} \norm{x_{t+1} - x_t} \\
        \tag{triangle in.Eq}
        &\le
        \frac12 \memLen^2 \sum_{t \in S}
            \norm{x_{t+1} - \bar{x}_{t+1}}
            +
            \norm{\bar{x}_{t+1} - \bar{x}_t}
            +
            \norm{\bar{x}_t - x_t} \\
        &\le
        \frac12 \memLen^2 \sum_{t = 1}^{\abs{S}} \pPredVar[t+1] + \pMeanShift + \pPredVar \\
        &\le
        \frac12 \memLen^2 \sum_{t = 1}^{\floor{T / \memLen}} \pMeanShift + 2\pPredVar,
    \end{align*}
    where the last transition also used the decreasing property of $\pPredVar$.
    Next, we have that for any $q > 0$
    \begin{align*}
        \sum_{t=\memLen}^{T} \sum_{i=1}^{\memLen} \norm{\bar{x}_{t+i-\memLen} - \bar{x}_{t+1-\memLen}}^q
        &=
        \sum_{t=\memLen}^{T} \sum_{i=2}^{\memLen} \norm{\sum_{j=1}^{i - 1} \bar{x}_{t+j+1-\memLen} - \bar{x}_{t+j-\memLen}}^q \\
        &=
        \sum_{i=2}^{\memLen} \sum_{j=1}^{i - 1} \sum_{t=\memLen}^{T} \norm{\bar{x}_{t+j+1-\memLen} - \bar{x}_{t+j-\memLen}}^q \\
        &\le
        \frac12 \memLen^2 \sum_{t=1}^{T-1} \norm{\bar{x}_{t+1} - \bar{x}_t}^q \\
        &\le
        \frac12 \memLen^2 \sum_{t=1}^{\abs{S}} \pMeanShift^q \\
        &\le
        \frac12 \memLen^2 \sum_{t=1}^{\floor{T / \memLen}} \pMeanShift^q,
    \end{align*}
    where the second transition follows since predictions change at most once every $\memLen$ rounds and thus there is at most one summand that is non-zero. This concludes the second part of the lemma.
    Next, recall that $t_i$ from \cref{lemma:intervalLengths} are the times \cref{alg:bcoReduction} updates the base BCO $\mathcal{A}$, and subsequently its prediction. Then we get that
    \begin{align*}
        \EEBrk{\sum_{t=\memLen}^{T} \sum_{i=1}^{\memLen} \norm{x_{t+i-\memLen} - \bar{x}_{t+i-\memLen}}^2}
        &\le
        \memLen \EEBrk{\sum_{t=1}^{T} \norm{x_{t} - \bar{x}_t}^2} \\
        &\le
        \memLen \EEBrk{\sum_{s=1}^{\abs{S}+1} \pPredVar[s]^2 \brk*{t_s - t_{s-1}}} \\
        &\le
        \memLen \sum_{s=1}^{\floor{T / \memLen}+1} \pPredVar[s]^2 \EEBrk{t_s - t_{s-1}} \\
        &\le
        3 \memLen^2 \sum_{t=1}^{\floor{T / \memLen}+1} \pPredVar^2.
    \end{align*}
    where the last two transitions used \cref{lemma:intervalLengths}.
\end{proof}

\section{Base BCO Algorithm}
\label{sec:baseBCOproofs}

We give a general example of a BCO algorithm that may be employed in conjunction with our reduction procedure given in \cref{alg:bcoReduction}.
For a positive semi-definite matrix $P \in \RR[d \times d]$ define the projection in $\norm{\cdot}_P$ distance
$
    \Pi_{\fConstraint}^P(x) = \argmin_{y \in \fConstraint} \norm{x - y}_P
    ,
$
where $\norm{x}_P^2 = x\tran P x$.
We analyze \cref{alg:zeroOMD}, a standard BCO procedure that uses a preconditioned gradient update, and a one-point gradient estimate.
\begin{algorithm}
	\caption{Base BCO} \label{alg:zeroOMD}
	\begin{algorithmic}[1]
	    \State \textbf{input:} regularization matrices $P_t \succeq 0$, step size $\eta$
		\State \textbf{set:} $\bar{x}_1 \in \fConstraint$
        
        \For{$t = 1, \ldots, T$}
            \State Draw $u_t \sim \sphere{d}$
            \State Play $x_t = \bar{x}_t + P_t^{-1/2} u_t$
            \State Observe $\bcoFeedback$ and set $\gHat = d \bcoFeedback P_t^{1/2} u_t$
            \State Update
            $
                \bar{x}_{t+1}
                =
                \Pi^{P_t}_{\fConstraint}\brk!{\bar{x}_t - \eta P_t^{-1} \gHat}
            $
        \EndFor
	\end{algorithmic}
\end{algorithm}

Since our setting of BCO with (no) memory $\memLen = 1$ uses a non-standard feedback model, we provide a full analysis of the bounds on the regret, and the prediction shifts and magnitudes. To that end, denote
\begin{align*}
    \constraintD = \max_{x,y \in \fConstraint} \norm{x-y},\qquad
    D_P = \max_{x,y \in \fConstraint} \norm{x-y}_{P},\qquad
    \bcoFeedbackBound = \max_{t \in \brk[s]{T}} \abs{\bcoFeedback}.
\end{align*}

\begin{lemma}
\label{lemma:generalBaseBCO}
    Consider the BCO with no memory ($\memLen=1$) setting described in \cref{sec:baseBCO} against an adversary that chooses $f_t : \fRange \to \RR$ that are $\pConvex-$strongly convex over $\fConstraint$ ($\pConvex = 0$ in the weakly convex case).
    If \cref{alg:zeroOMD} is run with regularization matrices $P_t = P_0 + \frac12 \pConvex \eta t I$ where $P_0 \succeq 0$, then
    \begin{align*}
    \pMeanShift
    =
    d \eta \bcoFeedbackBound \norm{P_t^{-1/2}}
    \qquad
    \pPredVar^2 = \norm{P_t^{-1}}
    \end{align*}
    satisfy the assumptions of $\baseBCOof{T}$ in \cref{thm:BCOreduction}.
    Moreover, for all $t \le T$ we have that
    \begin{enumerate}
        \item if $f_s$ are $\pLip$-Lipschitz then
        $
            \bcoMemRegret[1]{t}
            \le
            \frac{D_{P_1}^{2}}{\eta}
            +
            \frac{\eta d^2 \bcoFeedbackBound^2}{2} t
            +
            d \varepsilon t D_{P_t}
            +
            2\pLip \sum_{s=1}^{t} \norm{P_s^{-1/2}}
            ;
        $
        \item if $f_s$ are $\pSmooth$-smooth then
        $
            \bcoMemRegret[1]{t}
            \le
            \frac{D_{P_1}^{2}}{\eta}
            +
            \frac{\eta d^2 \bcoFeedbackBound^2}{2} t
            +
            d \varepsilon t D_{P_t}
            +
            \pSmooth \sum_{s=1}^{t} \norm{P_s^{-1}}
            .
        $
    \end{enumerate}
\end{lemma}
The proof of \cref{lemma:generalBaseBCO} relies on a few standard results.
First, we require a standard regret bound for the time-varying preconditioned update rule.
This is stated in the next lemma, which is is a specialization of bounds found in, e.g., \cite{duchi2011adaptive}, to the case of strongly convex quadratic regularizers.
\begin{lemma}%
\label{lemma:OMDcore}
    Let $\gHat[1], \ldots, \gHat[t] \in \RR[d]$, and $P_t \succeq \ldots \succeq P_1 \succ 0$ be arbitrary. For step size $\eta > 0$ define the update rule:
    $
        \bar{x}_{t+1} = \Pi_{\fConstraint}^{P_t} \brk*{\bar{x}_t - \eta P_t^{-1} \gHat[t]}.
    $
    Then we have that
    \begin{align*}
        \sum_{s=1}^{t} \gHat[s]\tran (\bar{x}_s - x^*)
        \le
        \frac{1}{\eta} \norm{\bar{x}_1 - x^*}_{P_1}^2
        +
        \frac{1}{\eta} \sum_{s=2}^{t} \norm{\bar{x}_s - x^*}_{P_s - P_{s-1}}^2
        +
        \frac{\eta}{2} \sum_{s=1}^{t} \norm{\gHat[s]}_{P_s^{-1}}^2
        ,\quad
        \forall x^* \in \fConstraint
        .
    \end{align*}
\end{lemma}
Next, we need the notion of smoothing and the one point-gradient estimate, which were initially proposed by \cite{flaxman2004online} and later refined in \cite{saha2011improved,hazan2014bandit}. The following lemma due to \cite{hazan2014bandit} encapsulates the relevant results.
\begin{lemma}[Lemmas 6 and 7 in \cite{hazan2014bandit}]
\label{lemma:fSmooth}
    Let $P \in \RR[d \times d]$ be symmetric and non-singular, $b \sim \ball{d}$, and $u \sim \sphere{d}$. Define the smoothed version of $f : \fRange \to \RR$ with respect to $P$ as
    \begin{equation*}
        \bar{f}(x) = \EEBrk{f\brk*{x + P b}}.
    \end{equation*}
    Then we have that:
    \begin{enumerate}[nosep,label=(\roman*)]
        \item $\nabla \bar{f}(x) = \EEBrk[0]{d f(x + P u)P^{-1} u}$;
        \item if $f$ is $\pConvex-$strongly convex then so is $\bar{f}$;
        \item if $f$ is convex and $\pSmooth-$smooth then
        $
            0
    		\le
    		\bar{f}(x) - f(x)
    		\le
    		\frac{\pSmooth}{2} \norm{P^2}
    		,\;
    		\forall x \in \fRange
    		;
        $
        \item if $f$ is convex and $\pLip$-Lipschitz then 
        $
            0
    		\le
    		\bar{f}(x) - f(x)
    		\le
    		\pLip \norm{P}
    		,\;
    		\forall x \in \fRange
    		.
        $
    \end{enumerate}
\end{lemma}
Among other things, this lemma implies that a regret bound for a sequence $\bar{f}_t$ yields one for $f_t$. We are now ready to prove \cref{lemma:generalBaseBCO}.

\begin{proof}[of \cref{lemma:generalBaseBCO}]
    First notice that $\bar{x}_t$ in \cref{alg:zeroOMD} is indeed the expectation of $x_t$ conditioned on all past history up to (not including) the decision at time $t$ (since $u_t$ is a zero mean independent random variable). Using the projection's shrinking property we get that
    \begin{align*}
        \norm{\bar{x}_{t+1} - \bar{x}_t}
        \le
        \norm{P_t^{-1/2}}
        \norm{\bar{x}_{t+1} - \bar{x}_t}_{P_t}
        \le
        \norm{P_t^{-1/2}}\norm{\eta d \bcoFeedback P_t^{-1/2} u_t}_{P_t}
        =
        \norm{P_t^{-1/2}} \eta d \abs{\bcoFeedback}
        \le
        d \eta \bcoFeedbackBound \norm{P_t^{-1/2}}
        =
        \pMeanShift.
    \end{align*}
    Next, we have that
    $
        \norm{x_t - \bar{x}_t}^2
        =
        \norm{P_t^{-1/2} u_t}^2
        \le
        \norm{P_t^{-1}}
        =
        \pPredVar^2,
    $
    thus concluding first part of the proof.
    Moving on to the regret bound, let $x \in \fConstraint$ be fixed, and denote $g_s = d f_s(\bar{x}_s) P_s^{1/2} u_s$, the desired gradient estimate at time $s$. Recalling that $\bar{x}_s$ is independent of the adversary's random variable $\xi_s$, we use total expectation to get that
    \begin{align*}
        \EEBrk{\sum_{s=1}^{t} (g_s - \gHat[s])\tran (\bar{x}_s - x)}
        &=
        d \EEBrk{\sum_{s=1}^{t} (f_s(\bar{x}_s) - \bcoFeedback[s])u_s\tran P_s^{1/2} (\bar{x}_s - x)} \\
        &=
        d \EEBrk{\sum_{s=1}^{t} (f_s(\bar{x}_s) - \EE[\xi_s]{\brk[s]{\bcoFeedback[s]}})u_s\tran P_s^{1/2} (\bar{x}_s - x)} \\
        &\le
        d \varepsilon \EEBrk{\sum_{s=1}^{t} \norm{\bar{x}_s - x}_{P_s}}
        \le
        d \varepsilon t D_{P_t},
    \end{align*}
    where the second to last transition used the Cauchy-Schwarz inequality, and the last transition used the assumption that $P_s$ is increasing.
    Next, notice that $\bar{x}_s, P_s, \gHat[s]$ satisfy the conditions of \cref{lemma:OMDcore}, and since
    $
        \norm{\gHat[s]}_{P_s^{-1}}^2
        \le
        d^2 \bcoFeedbackBound^2
        ,
    $
    we get that
    \begin{align*}
        \sum_{s=1}^{t} \gHat[s]\tran (\bar{x}_s - x)
        \le
        \frac{D_{P_1}^{2}}{\eta}
        +
        \frac{\pConvex}{2} \sum_{s=1}^{t} \norm{\bar{x}_s - x}^2
        +
        \frac{\eta d^2 \bcoFeedbackBound^2}{2} t,
    \end{align*}
    and taking expectation, summing the last two equations, and changing sides, we get that
    \begin{align}
    \label{eq:3}
        \EEBrk{
            \sum_{s=1}^{t} \bar{f}_s(\bar{x}_s)
            -
            \sum_{s=1}^{t} \bar{f}_s(x)
        }
        \le
        \EEBrk{
            \sum_{s=1}^{t} g_s\tran (\bar{x}_s - x)
            -
            \frac{\pConvex}{2} \sum_{s=1}^{t} \norm{\bar{x}_s - x}^2
        }
        \le
        \frac{D_{P_1}^{2}}{\eta}
        +
        \frac{\eta d^2 \bcoFeedbackBound^2}{2} t
        +
        d \varepsilon t D_{P_t}
        ,
    \end{align}
    where the first transition also used \cref{lemma:fSmooth} to show that $g_s$ is an unbiased estimate of $\nabla\bar{f}_s(\bar{x}_s)$ given $\bar{x}_s$, and that $\bar{f}_s$ are $\pConvex$ strongly convex (with $\pConvex = 0$ in the weakly convex case).
    Now, let $\bar{f}_s$ be smoothed with respect to $P_s^{-1/2}$ as defined in \cref{lemma:fSmooth}.
    If $f_s$ are $\pSmooth$ smooth, we get that
    \begin{align*}
        \EEBrk{f_s(x_s)}
        \le
        \EEBrk{f_s(\bar{x}_s) + \nabla f_s(\bar{x}_s)\tran P_s^{-1/2} u_s + \frac{\pSmooth}{2}\norm{P_s^{-1}}}
        =
        \EEBrk{f_s(\bar{x}_s)} + \frac{\pSmooth}{2}\norm{P_s^{-1}}
        \le
        \EEBrk{\bar{f}_s(\bar{x}_s)} + \frac{\pSmooth}{2}\norm{P_s^{-1}},
    \end{align*}
    where the last transition used \cref{lemma:fSmooth}, which also gives us that
    $
        -f_s(x)
        \le
        -\bar{f}_s(x) + \frac{\pSmooth}{2}\norm{P_s^{-1}}
    $. We thus conclude that
    \begin{align*}
        \bcoMemRegret[1]{t}
        =
        \max_{x \in \fConstraint}\brk[c]*{\EEBrk{
            \sum_{s=1}^{t} f_s(x_s)
            -
            \sum_{s=1}^{t} f_s(x)
        }}
        \le
        \max_{x \in \fConstraint}\brk[c]*{
        \EEBrk{
            \sum_{s=1}^{t} \bar{f}_s(\bar{x}_s)
            -
            \sum_{s=1}^{t} \bar{f}_s(x)
        }}
        +
        \pSmooth \sum_{s=1}^{t} \norm{P_s^{-1}}
        ,
    \end{align*}
    and plugging in \cref{eq:3} concludes the smooth case. Finally, if $f_s$ are $\pLip$ Lipschitz then using \cref{lemma:fSmooth} we get that
    \begin{align*}
        \EEBrk{f_s(x_s) - f_s(x)}
        \le
        \EEBrk{f_s(\bar{x}_s) - f_s(x)} + \pLip\norm{P_s^{-1/2}}
        \le
        \EEBrk{\bar{f}_s(\bar{x}_s) - \bar{f}_s(x)} + 2\pLip\norm{P_s^{-1/2}},
    \end{align*}
    and summing over $s$ and plugging \cref{eq:3} concludes the non-smooth case.
\end{proof}

\section{Main Result Proofs}

\subsection{Proof of \cref{lemma:baseBCO}}
\label{sec:proofOfBaseBCOlemma}
This lemma is a direct specification of \cref{lemma:generalBaseBCO} with the appropriate choice of parameters.
For $M \in \MRange$, denote $\vStack{M} \in \RR[d k \memLen]$ the column stacking of $M$. Next, denote $I \in \RR[d k \times d k]$, the identity matrix, and $\mathrm{diag}(\rtOf{1}, \ldots, \rtOf{\memLen}) \in \RR[\memLen \times \memLen]$, the diagonal matrix with $\rt \in \RR[\memLen]$ on its diagonal. Consider \cref{alg:zeroOMD} with $\fConstraint = \MConstraint$ (column stacked), dimension $\dM = d k \memLen$, and
\begin{align*}
    P_t
    =
    \brk[s]{
        \mathrm{diag}(\rtOf{1}, \ldots, \rtOf{\memLen})
        \otimes
        I
    }^{-2},
\end{align*}
where $\otimes$ is the Kronecker product. Then, first, for $M \in \MRange$ we can write the projection as
\begin{align*}
    \Pi_{\MConstraint}^{P_t}(\vStack{M})
    =
    \argmin_{M' \in \MConstraint} 
    \norm{\vStack{M} - \vStack{M'}}_{P_t}^2
    =
    \argmin_{M' \in \MConstraint} 
    \sum_{i=1}^{\memLen} \brk{\rti}^{-2} \norm{M^{\brk[s]{i}} - {M'}^{\brk[s]{i}}}_F^2
    ,
\end{align*}
and since $\MConstraint = \MiConstraint[1] \times \ldots \times \MiConstraint[\memLen]$, each term in the sum may be minimized separately, and so we get that
\begin{align*}
    \Pi_{\MConstraint}^{P_t}(\vStack{M})
    =
    \vStack{
    \Pi_{\MiConstraint[1]}(M^{\brk[s]{1}}), \ldots, \Pi_{\MiConstraint[\memLen]}(M^{\brk[s]{\memLen}})
    }
    ,
\end{align*}
where
$
    \Pi_{\MiConstraint}(M)
    =
    \argmin_{M' \in \MiConstraint} \norm{M - M'}
    .
$
Second, we have that
$
    \gHat
    =
    \vStack{
        \gHat^{\brk[s]{1}}, \ldots, \gHat^{\brk[s]{\memLen}} 
    }
$
and thus the update rule may be rewritten as
\begin{align*}
    \vStack{\barMt[t+1]}
    =
    \vStack{\barMt[t]}
    -
    \eta P_t^{-1} \gHat[t]
    =
    \vStack{
        \barMtOf[t]{1} - \eta \brk{\rtOf[t]{1}}^2 \gHat^{\brk[s]{1}},
        \ldots,
        \barMtOf[t]{\memLen} - \eta \brk{\rtOf[t]{\memLen}}^2 \gHat^{\brk[s]{\memLen}}
    }
    .
\end{align*}
We conclude that the procedure in \cref{lemma:baseBCO} is indeed described by \cref{alg:zeroOMD}. We can now conclude the lemma using \cref{lemma:generalBaseBCO}. A simple calculation shows that
\begin{align*}
    \constraintD^2
    &=
    \max_{M_1, M_2 \in \MConstraint} \norm{M_1 - M_2}_F^2
    \le
    4 n^2 \kappa_B^2 \kappa^6 / \gamma ;
    \\
    \constraintDp^2
    &=
    \max_{M_1, M_2 \in \MConstraint} \norm{M_1 - M_2}_{P_0}^2
    \le
    \memLen n^2.
\end{align*}
Moreover, $D_{P_t}^2 = D_{P_0}^2 + \frac12 \pConvex \eta t D^2$, and $\norm{P_t^{-1}} \le \frac{2}{\pConvex \eta t}$.
Plugging this into \cref{lemma:generalBaseBCO} we get that for all $t \le T$
\begin{align*}
    \bcoMemRegret[1]{t}
    \le
    &\frac{1}{\eta} \brk*{\memLen n^2 + \frac{2\pSmoothF}{\pConvexF}(1 + \log T)}
    +
    \frac{\dM^2 \bcoFeedbackBound^2}{2} \eta T \\
    &+
    \underbrace{
    \frac{2 n^2 \kappa_B^2 \kappa^6 \pConvexF}{\gamma}
    +
    \dM \varepsilon \brk*{
        \sqrt{\memLen} n T
        +
        \sqrt{\frac{4 \pConvexF \eta n^2 \kappa_B^2 \kappa^6 T^3}{\gamma}}}
    }_{R_{\text{low}}}
    ,
\end{align*}
and for $\varepsilon \in \tilde{O}(T^{-1})$, and $\eta \in \tilde{O}(T^{-1/2})$, we indeed have that $R_{\text{low}} \in \tilde{O}(T^{-1/4})$. Finally, $\pMeanShift, \pPredVar$ translate directly between lemmas, thus concluding the proof.

\subsection{Low Order Terms in \cref{thm:main}}
\label{sec:mainproofs}
We summarize the low order terms that were omitted in the last three equations of the proof of \cref{thm:main} given in \cref{sec:mainAlgProof}.
The first of the three explicitly states the lower order term
\begin{align*}
    R_{\text{low}}^{(1)}
    =
    2 \Dxu^2 (\cG + \memLen \cC)
    ,
\end{align*}
which is later omitted in the last step.
The second equation results from invoking \cref{lemma:baseBCO} with $\bcoFeedbackBound = \cC \Dxu^2$ and horizon $T / 2(\memLen+1)$, to bound the second term of \cref{thm:BCOreduction}. Here the terms related to $\pMeanShift, \pMeanShift^2$ were omitted, and satisfy
\begin{align*}
    \sum_{t=1}^{\floor{T / 2(\memLen+1)}+1} \brk{\pLipF \pMeanShift + \pSmoothF \pMeanShift^2}
    \le
    2\pLipF d k \cC \Dxu^2 \sqrt{\frac{\eta \memLen T}{\pConvexF}}
    +
    \frac{2 \pSmoothF d^2 k^2 \memLen^2 \cC^2 \Dxu^4 \eta \log T}{\pConvexF}
    =
    R_{\text{low}}^{(2)}
    .
\end{align*}
The third equation results from plugging in the previous result as well as that of \cref{lemma:baseBCO} with horizon $T / (2(\memLen+1))$, and $\varepsilon = \cG \Dxu^2 / T$ into \cref{thm:BCOreduction}. \cref{lemma:baseBCO} yields a low order term, which is given at the end of the proof in \cref{sec:proofOfBaseBCOlemma}. Plugging in the horizon, $\varepsilon$, and $\bcoFeedbackBound$ this term is given by
\begin{align*}
    R_{\text{low}}^{(3)}
    =
    \frac{2 n^2 \kappa_B^2 \kappa^6 \pConvexF}{\gamma}
    +
    d^2 k^2 \memLen^2  \cG \Dxu^2 \brk*{
        \sqrt{\memLen} n
        +
        \sqrt{\frac{2 \pConvexF \eta n^2 \kappa_B^2 \kappa^6 T}{\gamma(\memLen+1)}}},
\end{align*}
and thus the final low order term is given by
\begin{align*}
    R_{\text{low}}
    =
    R_{\text{low}}^{(1)}
    +
    2(\memLen + 1)^2 R_{\text{low}}^{(2)}
    +
    6 (\memLen + 1) R_{\text{low}}^{(3)}.
\end{align*}
Since $\memLen$ is logarithmic in $T$, and $\eta \in \tilde{O}(T^{-1/2})$, we get that $R_{\text{low}} \in \tilde{O}(T^{1/4})$, as desired.

\subsection{Proof of \ref{item:smoothSurrogate} in \cref{lemma:SurrogateProperties}} \label{sec:fSmoothProof}

\begin{proof}
    Recall from \cref{def:surrogateCosts} that
    $
        f_t(M_{0:H})
        =
        c_t(y_t(M_{0:H}), v_t(M_{0:H}))
        ,
    $
    and denote 
    \begin{align*}
        z_t(M_{0:H})
        =
        \brk[s]{y_t(M_{0:H})\tran v_t(M_{0:H})\tran}\tran
        .
    \end{align*}
    Since $z_t(\cdot)$ is a linear mapping, its Jacobian is constant, and we denote it as $J_{z_t}$. Applying the chain rule, we get that for all $M_0, \ldots, M_H \in \MRange$
    \begin{align*}
        \norm{\nabla^2 f_t(M_{0:H})}
        =
        \norm{J_{z_t}\tran \nabla^2 c_t(y_t(M_{0:H}), v_t(M_{0:H})) J_z}
        \le
        \pSmooth \norm{J_{z_t}}^2,
    \end{align*}
    and thus bounding $\norm{J_{z_t}}$ will show that $f_t$ is smooth.
    To that end, notice that an intermediate step of Lemma 5.6 of \cite{agarwal2019online} shows that for any $M_0, \ldots, M_H, {M'}_0, \ldots, {M'}_H \in \MRange$, we have that
    \begin{align}
    \label{eq:zLipBase}
        \norm{z_t(M_{0:H}) - z_t(M_0, \ldots, {M'}_{H-k}, \ldots, M_H)}
        \le
        5 \kappa_B \kappa^3 W (1 - \gamma)^k \sum_{i=0}^H \norm{M_{H-k}^{\brk[s]{i}} - {M'}_{H-k}^{\brk[s]{i}}}.
    \end{align}
    Recalling that $\norm{\cdot} \le \norm{\cdot}_F$, and using the triangle and Cauchy-Schwarz inequalities we get that
    \begin{align*}
        \norm{z_t(M_{0:H}) - z_t({M'}_{0:H})}
        &\le
        5 \kappa_B \kappa^3 W
        \sum_{k=0}^{H} (1 - \gamma)^k \sum_{i=1}^H \norm{M_{H-k}^{\brk[s]{i}} - {M'}_{H-k}^{\brk[s]{i}}}_F \\
        &\le
        5 \kappa_B \kappa^3 W \sqrt{H}
        \sum_{k=0}^{H} (1 - \gamma)^k \norm{M_{H-k} - {M'}_{H-k}}_F \\
        &\le
        5 \kappa_B \kappa^3 W \sqrt{H} \norm{M_{0:H} - {M'}_{0:H}}_F
        \sqrt{\sum_{k=0}^{H} (1 - \gamma)^{2k}} \\
        &\le
        5 \kappa_B \kappa^3 W \sqrt{\frac{H}{\gamma}} \norm{M_{0:H} - {M'}_{0:H}}_F
        .
    \end{align*}
    Since $\MRange$ contains an open set of ${\RR[(k \times d)\times (\memLen+1)]}$, this Lipschitz property implies that
    $
        \norm{J_z}
        \le
        5 \kappa_B \kappa^3 W \sqrt{\frac{H}{\gamma}}
        \le 
        \sqrt{\pSmoothF / \pSmooth}
        ,
    $
    thus showing that $f_t$ is $\pSmoothF$ smooth. Since $F_t$ results from taking expectation of $f_t$ with respect to the random system noise, it is also $\pSmoothF$ smooth.
    
    Next, recall that $\tilde{f}_t(M) = f_t(M, \ldots, M)$, and thus defining $\tilde{z}_t = z_t(M, \ldots, M)$, and repeating the process above, it suffices to show that $\tilde{z}_t$ is $\sqrt{\pSmoothF / \pSmooth}$ Lipschitz to conclude that $\tilde{f}_t, \tilde{F}_t$ are $\pSmoothF$ smooth. Using \cref{eq:zLipBase} we get that for $M, M' \in \MRange$
    \begin{align*}
        \norm{\tilde{z}_t(M) - \tilde{z}_t({M'})}
        &\le
        5 \kappa_B \kappa^3 W
        \sum_{k=0}^{H} (1 - \gamma)^k \sum_{i=1}^H \norm{M^{\brk[s]{i}} - {M'}^{\brk[s]{i}}}_F \\
        &\le
        5 \kappa_B \kappa^3 W \sqrt{H} \norm{M - {M'}}_F
        \sum_{k=0}^{H} (1 - \gamma)^k \\
        &\le
        \frac{5 \kappa_B \kappa^3 W}{\gamma} \sqrt{\memLen} \norm{M - {M'}}_F \\
        &=
        \sqrt{\pSmoothF / \pSmooth} \norm{M - {M'}}_F
        ,
    \end{align*}
    thus establishing the Lipschitz property and concluding the proof.
\end{proof}

\section{Extensions Proofs}
\label{sec:extensionAnalysis}

We first need to extend the base BCO procedure to the weakly convex cases. Similarly to \cref{lemma:baseBCO}, this is an immediate corollary of \cref{lemma:generalBaseBCO} with appropriate choice of parameters.

\begin{lemma} \label{lemma:baseBCOextended}
    Consider the setting of \cref{sec:bcoMem} with $H=1$ and $\varepsilon \in \tilde{O}(1/T)$, against an adversary that chooses $f_t: \MRange \to \RR$ that are convex. Let $\dM, \constraintD, \rti[0]$ be as in \cref{thm:mainExtended}, and $\bcoMemRegret[1]{t}$ be the regret of a procedure that at time $t$:
    \begin{enumerate}[nosep,label=(\roman*)]
        \item Draws $\Ut[t] \sim \sphere{(k \times d) \times \memLen}$; and plays $\Mt$ where $\Mti[t] = \barMti[t] + \rti[t]\Uti[t]$ $(\forall i \in \brk[s]{\memLen})$
        \item Observes $\bcoFeedback$; and sets $\gHat^{\brk[s]{i}} = (d_{\MConstraint} / \rti[t]) \bcoFeedback \Uti[t]$ $(\forall i \in \brk[s]{\memLen})$
        \hfill (1-point gradient estimate)
        \item Updates
            $
                \barMti[t+1] 
                =
                \Pi_{\MiConstraint}\brk[s]!{\barMti[t] - \eta  \brk{\rti[t]}^2 \gHat^{\brk[s]{i}}}
                \quad
                (\forall i \in \brk[s]{\memLen})
                .
            $
            \hfill (preconditioned update)
    \end{enumerate}
    Suppose that $\abs{\bcoFeedback} \le \bcoFeedbackBound$ then for all $t \le T$:
    \begin{enumerate}
        \item if $f_t$ are $\pLip$-Lipschitz, 
        $
            \eta 
            =
            2 \brk[s]*{\frac{\pLip^2 \constraintD^2}{\dM^6 \constraintD^6 T}}^{1/4}
        $,
        and
        $
            \rti 
            =
            \brk[s]*{
            (\rti[0])^{-2}
            +
            \frac{4 \pLip \sqrt{T}}{\dM \bcoFeedbackBound \constraintD}
            }^{-1/2}
        $
        then 
        \begin{align*}
            \bcoMemRegret[1]{t}
            \le
            4 \sqrt{\dM \pLip \constraintD \bcoFeedbackBound} T^{3/4}
            +
            \tilde{O}(T^{1/4})
            ,
            \qquad
            \pMeanShift
            =
            \frac{\constraintD}{\sqrt{T}}
            ,
            \qquad
            \pPredVar^2 = \frac{\dM \constraintD \bcoFeedbackBound}{4 \pLip \sqrt{T}}
            ;
        \end{align*}
        \item if $f_t$ are $\pSmoothF$ smooth, 
        $
            \eta
            =
            \brk[s]*{
            \frac{2 \pSmoothF \constraintD^2}{\dM^4 \bcoFeedbackBound^4 T}
            }^{1/3}
        $,
        and
        $
            \rti 
            =
            \brk[s]*{
            (\rti[0])^{-2} 
            +
            \brk{\frac{4 \pSmoothF^2 T}{\dM^2 \bcoFeedbackBound^2 \constraintD^2}}^{1/3}
            }^{-1/2}
        $
        then
        \begin{align*}
            \bcoMemRegret[1]{t}
            \le
            \brk{4 \sqrt{\pSmoothF} \dM \bcoFeedbackBound \constraintD T}^{2/3} 
            +
            \tilde{O}(T^{1/3})
            ,
            \qquad
            \pMeanShift
            =
            \frac{\constraintD}{\sqrt{T}}
            \qquad
            \pPredVar^2 
            =
            \brk*{\frac{\dM^2 \bcoFeedbackBound^2 \constraintD^2}{4 \pSmoothF^2 T}}^{1/3}
            .
        \end{align*}
    \end{enumerate}
\end{lemma}
See proof in \cref{sec:proofOfExtendedBaseBCO}.

\begin{proof}[of \cref{thm:mainExtended}]
First, unlike the proof of \cref{thm:main}, here we use $f_t$ as given in \cref{def:surrogateCosts}, without any modification.
As before we view \cref{alg:main} in the context of the BCO with memory setting presented in \cref{sec:bcoMem}. The adversary's noise $\xi_t$ is now degenerate ($w_t$ are not stochastic), the costs are given by $f_t$ and by \cref{lemma:controlToOCO}, the feedback satisfies
\begin{align*}
    \abs{
        \ct{x_t}{u_t} - f_t(M_{t-1-2\memLen:t})
    }
    \le
    \frac{\cG \Dxu^2}{T},
\end{align*}
and thus $\varepsilon = {\cG \Dxu^2}/{T}$, and the feedback $\ct{\xt}{\ut}$ is bounded by $\cC \Dxu^2$. Let $\bcoMemRegret[\memLen+1]{T}$ be the regret of \cref{alg:main} against an adversary with memory $\memLen + 1$, and notice that our choice of $\rti$, and in particular $\rti[0]$, ensures that $M_t \in \MRange$. Since the third part of \cref{lemma:controlToOCO} is, in fact, proven for $f_t$ rather than $F_t$ (the latter is an immediate corollary) and so we have that
\begin{align*}
    \controlRegret
    \le
    \bcoMemRegret[\memLen+1]{T}
    +
    2 \Dxu^2 (\cG + H \cC)
    .
\end{align*}
Since the second term is at most poly-log in $T$, it remains to bound $\bcoMemRegret[\memLen+1]{T}$ for each set of assumptions and parameter choices.
Recall that \cref{alg:main} fits the mold of our reduction procedure given in \cref{alg:bcoReduction} with base procedure as in \cref{lemma:baseBCOextended}.
Moving forward, our analysis is divided in two. Consider the first set of parameter choices (with no smoothness assumptions). By \cref{lemma:SurrogateProperties}, $f_t$ are coordinate-wise $\pLipF$ Lipschitz, and thus $\tilde{f}_t : M \mapsto f_t(M,\ldots, M)$ are $(\memLen+1)\pLipF$ Lipschitz. Invoking \cref{lemma:baseBCOextended} with 
$
\bcoFeedbackBound 
=
\cC \Dxu^2
,
\pLip
=
(\memLen+1)\pLipF
$
and horizon $T / (\memLen+1)$, the second term of \cref{thm:BCOreduction} (with no smoothness assumption) satisfies that
\begin{align*}
    \frac12 \pLipF (\memLen + 1)^2 \sum_{t=1}^{\floor{T / (\memLen+1)}} \pMeanShift +  2 \pPredVar
    &\le
    \frac12 \pLipF (\memLen + 1) T (\pMeanShift +  2 \pPredVar)
    \\
    &\le
    \frac12 \pLipF (\memLen + 1) T \brk*{\constraintD \sqrt{\frac{\memLen+1}{T}} +  \sqrt{\frac{\dM \bcoFeedbackBound \constraintD}{\pLipF}} \brk*{\frac{\memLen+1}{T}}^{1/4}}
    \\
    &\le
    \frac12 \sqrt{\dM \bcoFeedbackBound \constraintD \pLipF (\memLen+1)^{5/2}} T^{3/4}
    +
    \tilde{O}(T^{1/2})
    ,
\end{align*}
and further using \cref{lemma:baseBCOextended} to bound the first term of \cref{thm:BCOreduction}, and simplifying, we get that
\begin{align*}
    \bcoMemRegret[\memLen+1]{T}
    &\le
    13 \sqrt{\dM \bcoFeedbackBound \constraintD \pLipF (\memLen+1)^{5/2}} T^{3/4}
    +
    \tilde{O}(T^{1/2}) \\
    &\le
    13 \sqrt{2 d k n \cC \Dxu^2 \kappa_B \kappa^3 \gamma^{-1/2} \pLipF (\memLen+1)^{7/2}} T^{3/4}
    +
    \tilde{O}(T^{1/2})
    ,
\end{align*}
where the last step only plugs in the values of $\dM, \bcoFeedbackBound, \constraintD$. This concludes the proof of the non-smooth case.
Now suppose that $c_t$ are $\pSmooth$ smooth and \cref{alg:main} is run with our second choice of parameters. Notice that the proof of the smoothness in \cref{lemma:SurrogateProperties} (see \cref{sec:fSmoothProof}) actually shows that both $f_t, \tilde{f}_t$ are $\pSmoothF$ smooth, and as before, we invoke \cref{lemma:baseBCOextended} with our parameter choices to bound the second term of \cref{thm:BCOreduction} (with the smoothness assumption) by
\begin{align*}
    \frac12 (\memLen+1)^2
    \sum_{t=1}^{\floor{T / 2(\memLen+1)}+1} 
    &\brk{(\memLen+1)\pLipF \pMeanShift + \pSmoothF \pMeanShift^2 + 6 \pSmoothF \pPredVar^2}
    \le
    (\memLen + 1) T \brk{(\memLen+1)\pLipF \pMeanShift + \pSmoothF \pMeanShift^2 + 6 \pSmoothF \pPredVar^2}
    \\
    &\le
    (\memLen + 1) T \brk*{
        \pLipF \constraintD \sqrt{\frac{(\memLen+1)^3}{T}}
        + 
        \pSmoothF \constraintD^2 {\frac{\memLen+1}{T}}
        +
        6 \pSmoothF \brk*{\frac{\dM^2 \bcoFeedbackBound^2 \constraintD^2 (\memLen+1)}{4 \pSmoothF^2 T}}^{1/3}
    }
    \\
    &\le
    4 \brk*{\sqrt{\pSmoothF} \dM \bcoFeedbackBound \constraintD (\memLen+1)^2 T}^{2/3}
    +
    \tilde{O}(T^{1/2})
    ,
\end{align*}
and further using \cref{lemma:baseBCOextended} to bound the first term of \cref{thm:BCOreduction}, and simplifying, we get that
\begin{align*}
    \bcoMemRegret[\memLen+1]{T}
    &\le
    12 \brk*{\sqrt{\pSmoothF} \dM \bcoFeedbackBound \constraintD (\memLen+1)^2 T}^{2/3}
    +
    \tilde{O}(T^{1/2})
    \\
    &\le
    12 \brk*{2 d k n \cC \Dxu^2 \kappa_B \kappa^3 \sqrt{\pSmoothF / \gamma} (\memLen+1)^3 T}^{2/3}
    +
    \tilde{O}(T^{1/2})
    ,
\end{align*}
where the last step only plugs in the values of $\dM, \bcoFeedbackBound, \constraintD$.
\end{proof}

\subsection{Proof of \cref{lemma:baseBCOextended}}
\label{sec:proofOfExtendedBaseBCO}

    Recall \cref{sec:proofOfBaseBCOlemma}, where we show that \cref{lemma:baseBCO} is a direct corollary of \cref{lemma:generalBaseBCO}. As the procedure itself does not change here, the proof is concluded by plugging-in our assumptions and parameter choices into \cref{lemma:generalBaseBCO}.
    For the first case, $f_t$ are $\pLip$ Lipschitz, and our choice of parameters gives that
    \begin{align*}
        D_{P_t}^2
        =
        D_{P_1}^2
        =
        D_{P_0}^2 + \frac{4 \pLip \constraintD \sqrt{T}}{\dM \bcoFeedbackBound}
        =
        \memLen n^2 + \frac{4 \pLip \constraintD \sqrt{T}}{\dM \bcoFeedbackBound},\\
        \norm{P_s^{-1/2}}
        \le
        \rtOf[1]{1}
        \le
        \frac12 \sqrt{\frac{\constraintD \dM \bcoFeedbackBound}{\pLip \sqrt{T}}},
    \end{align*}
    and plugging into \cref{lemma:generalBaseBCO} we get that
    \begin{align*}
        \bcoMemRegret[1]{t}
        \le
        4 \sqrt{\dM \pLip \constraintD \bcoFeedbackBound} T^{3/4}
        +
        \underbrace{
            \frac{\memLen n^2}{\eta}
            +
            \dM \varepsilon T D_{P_1}
        }_{R_{\text{low}}}
        ,
    \end{align*}
    and by our assumptions we indeed get that $R_{\text{low}} \in \tilde{O}(T^{1/4})$ as desired. To conclude the non-smooth case, we further apply \cref{lemma:generalBaseBCO} to get that
    \begin{align*}
        \pMeanShift
        =
        \frac{\constraintD}{\sqrt{T}}
        \qquad
        \pPredVar^2 = \frac{\dM \constraintD \bcoFeedbackBound}{4 \pLip \sqrt{T}}
        .
    \end{align*}
    Next, for the second case, $f_t$ are $\pSmoothF$ smooth and our choice of parameters gives that
    \begin{align*}
        D_{P_t}^2
        =
        D_{P_1}^2
        =
        D_{P_0}^2 
        +
        \constraintD^2
        \brk*{\frac{4 \pSmoothF^2 T}{\dM^2 \bcoFeedbackBound^2 \constraintD^2}}^{1/3}
        \le
        \memLen n^2 
        +
        \brk*{\frac{4 \pSmoothF^2 \constraintD^4 T}{\dM^2 \bcoFeedbackBound^2}}^{1/3},
        \\
        \norm{P_s^{-1}}
        \le
        \brk{\rtOf[1]{1}}^2
        \le
        \brk*{\frac{\dM^2 \bcoFeedbackBound^2 \constraintD^2}{4 \pSmoothF^2 T}}^{1/3},
    \end{align*}
    and plugging into \cref{lemma:generalBaseBCO} we get that
    \begin{align*}
        \bcoMemRegret[1]{t}
        \le
        \brk{4 \sqrt{\pSmoothF} \dM \bcoFeedbackBound \constraintD T}^{2/3} 
        +
        \underbrace{
            \frac{\memLen n^2}{\eta}
            +
            \dM \varepsilon T D_{P_1}
        }_{R_{\text{low}}}
        ,
    \end{align*}
    and by our assumptions we indeed get that $R_{\text{low}} \in \tilde{O}(T^{1/3})$ as desired. To conclude the smooth case, and thus the proof, we further apply \cref{lemma:generalBaseBCO} to get that
    \begin{align*}
        \pMeanShift
        =
        \frac{\constraintD}{\sqrt{T}}
        \qquad
        \pPredVar^2 
        =
        \brk*{\frac{\dM^2 \bcoFeedbackBound^2 \constraintD^2}{4 \pSmoothF^2 T}}^{1/3}
        ,
    \end{align*}
    as desired.

\end{document}